\documentclass{article}

\usepackage{algorithm,algorithmic}
\usepackage{microtype}
\usepackage{graphicx}
\usepackage{subfigure}
\usepackage{booktabs} 
\usepackage{bbm,bm}
\usepackage{epsfig,ifthen}
\usepackage{latexsym}
\usepackage{amsfonts}
\usepackage{eepic}
\usepackage{amsopn}
\usepackage{amsmath,amssymb,amsthm,rotating,graphicx,rotating,float}
\usepackage{accents}
\usepackage{cite}
\usepackage[mathscr]{eucal}
\usepackage{url}
\usepackage{graphicx}
\usepackage{csquotes}
\usepackage{verbatim}
\usepackage{ulem}
\usepackage{color} 

\newlength{\dhatheight}

\newcommand{\argmax}{\mathop{\rm argmax}}
\newcommand{\argmin}{\mathop{\rm argmin}}

\newcommand{\ignore}[1]{}
\newcommand{\todo}[1]{}
\newcommand{\oldstuff}[1]{}

\newsavebox{\savepar}

\makeatletter
\newcommand{\vast}{\bBigg@{3}}
\newcommand{\Vast}{\bBigg@{4}}
\makeatother

\newcommand{\eat}[1]{}
\newtheorem{definition}{Definition}[section]
\newtheorem{theorem}{Theorem}

\newtheorem{proposition}[theorem]{Proposition}
\usepackage{hyperref}

\title{TIP: Typifying the Interpretability of Procedures\footnote{All authors have affiliation to IBM Research.}}

\author{Amit Dhurandhar, Vijay Iyengar, Ronny Luss and Karthikeyan Shanmugam}

\begin{document}
\maketitle

\begin{abstract}
We provide a novel notion of what it means to be interpretable, looking past the usual association with human understanding. Our key insight is that interpretability is not an absolute concept and so we define it relative to a target model, which may or may not be a human. We define a framework that allows for comparing interpretable procedures by linking them to important practical aspects such as accuracy and robustness. We characterize many of the current state-of-the-art interpretable methods in our framework portraying its general applicability. Finally, principled interpretable strategies are proposed and empirically evaluated on synthetic data, as well as on the largest public olfaction dataset that was made recently available \cite{olfs}. We also experiment on MNIST with a simple target model and different oracle models of varying complexity. This leads to the insight that the improvement in the target model is not only a function of the oracle model's performance, but also its relative complexity with respect to the target model. Further experiments on CIFAR-10, a real manufacturing dataset and FICO dataset showcase the benefit of our methods over Knowledge Distillation when the target models are simple and the complex model is a neural network.
\end{abstract}

\section{Introduction}
\label{intro}
What does it mean for a model to be interpretable? From our human perspective, interpretability typically means that the model can be explained, a quality which is imperative in almost all real applications where a human is responsible for consequences of the model. However good a model might have performed on historical data, in critical applications, interpretability is necessary to justify, improve, and sometimes simplify decision making. 

Understanding complex models, however, has two parts. One is providing understandable explanations of its action/prediction on specific cases -``Why did the model act this way on this sample?" These explanations are local to a specific decision on a sample and constitute model explanability \cite{finale,unifiedPI}. However, there is also the seemingly complementary question - ``What useful insight can the model \textit{as a whole} provide to the end user?" Usually this question \cite{bastani2017interpreting,Caruana:2015,irt} is associated with enhancing a broad understanding about the behavior of the model. We in this work define `interpretability' primarily with regards to the second question, while also allowing for interpretations of local behaviors of models that address important aspects of the first question. Note that the two outlined questions, although different, are not mutually exclusive.

A great example of this is a malware detection neural network \cite{malwarecom} which was trained to distinguish regular code from malware. The neural network had excellent performance, presumably due to the deep architecture capturing some complex phenomenon opaque to humans, but it was later found that the primary distinguishing characteristic was the grammatical coherance of comments in the code, which were either missing or written poorly in the malware as opposed to regular code. In hindsight, this seems obvious as you wouldn't expect someone writing malware to expend effort in making it readable. This example shows how the interpretation of a seemingly complex model can aid in creating a simple rule. Please note that one could piece up this summary by looking at model explanations on many example code snippets and inferring after the fact that all factors that contributed to the decisions are related to code comments. However, the essence of interpretability here is about directly finding a simple global rule that captures model behavior as a whole without compromising its  performance.

The above example defines interpretability as humans typically do: we require the model to be understandable. This thinking would lead us to believe that, in general, complex models such as random forests or even deep neural networks are not interpretable. However, just because we cannot always understand what the complex model is doing does not necessarily mean that the model is not interpretable in some other useful sense. It is in this spirit that we define the novel notion of $\delta$-interpretability that is more general than being just interpretable relative to a human. The need for such formalisms was echoed in \cite{lipton2016mythos}, where the author stresses the need for a proper formalism for the notion of interpretability and quantifying methods based on these formalisms. 

Given this, our contributions are two-fold. We first formalize the notion of $\delta$-interpretability by building a general framework under which many existing works fall, as well as extend our definitions to allow for aspects such as robustness.  This new framework is then used as motivation to derive new interpretable procedures and we illustrate their usefulness on various real and synthetic datasets. More specifically, our contributions (and the outline of the paper) are as follows:
\begin{enumerate}
\item Framework Description: Sections 2 - 6
\begin{itemize}
\item Sections 2 and 3: We provide a formal definition of relative interpretability with respect to a target model (TM). It is based on the improvement (or degradation) in the performance of the target model on a task that is brought about by an interpretable procedure communicating information from a more complex model (CM). The key notion is that the target model class remains invariant in this process.  We also specifically address how this is tied to human interpretability.

\item Section 4: We showcase the flexibility of our definition and how it can be easily extended to account for other practical aspects such as robustness of models in finite sample settings. Moreover, we prove how our extended definition reduces to the original one in the ideal setting where we have access to the target data distribution.

\item Sections 5 and 6: We show how several existing state of the art works on interpretability can be cast in our general framework. Moreover, we argue how our framework extends beyond cases where we seemingly have a well-defined goal.
\end{itemize}
\item Methods and Experiments: Section 7
\begin{itemize}
\item Section 7.1: We propose new interpretable procedures that involves weighting by confidence scores as a means to transfer information from a highly accurate complex model to a target model with a priori low accuracy. We derive error bounds for the target model to theoretically ground this procedure.

\item Sections 7.2 and 7.3: We apply our interpretable procedure on synthetic as well as on a real life Olfaction dataset where our procedure greatly improves an interpretable Lasso model using information from the complex model (Random Forests) with superior performance. Moreover, we describe how insights from this improved lasso model has led to further investigations by human experts.

\item Section 7.4: We show that the most complex model need not be the best with respect to a fixed interpretable procedure aligning with some everyday intuition about student-teacher relationships. We demonstrate this using experimental results on the MNIST dataset.

\item Sections 7.5 and 7.6: We further compare our methods with Knowledge Distillation on CIFAR-10 dataset, a real industrial manufacturing dataset and the FICO dataset.
\end{itemize}
\end{enumerate}

In our framework, our target model could be a human where performance in a specific task is measured after interaction of the human and the model in a typical human study. In a more general sense, our target model could be something that is regarded as immediately interpretable by a human. Therefore, improving performance \textit{while retaining} the model class complexity can directly contribute to human understanding. This is the key idea behind our framework. Our framework focuses on measurable global/local insights conveyed by a complex model to the target model that is also cognizant of its usefulness.

We offer an example from the healthcare domain \cite{chang2010}, where interpretability is a critical modeling aspect, as a running example in our paper. The task is predicting future costs based on demographics and past insurance claims (including doctor visit costs, justifications, and diagnoses) for members of the population. 
The data used in \cite{chang2010} represents diagnoses using ICD-9-CM (International Classification
of Diseases) coding which had on the order of 15,000 distinct codes at the time of the study. The high dimensional nature of diagnoses led to the development of various abstractions such as the ACG (Adjusted Clinical Groups) case-mix system \cite{starfield1991}, which output various mappings of the ICD codes to lower dimensional categorical spaces, some even independent of disease. A particular mapping of IDC codes to 264 Expanded Diagnosis Clusters (EDCs) was used in \cite{chang2010} to create a complex model that performed quite well in the prediction task. 

\begin{figure*}[t]
\centering
   \includegraphics[width=\linewidth]{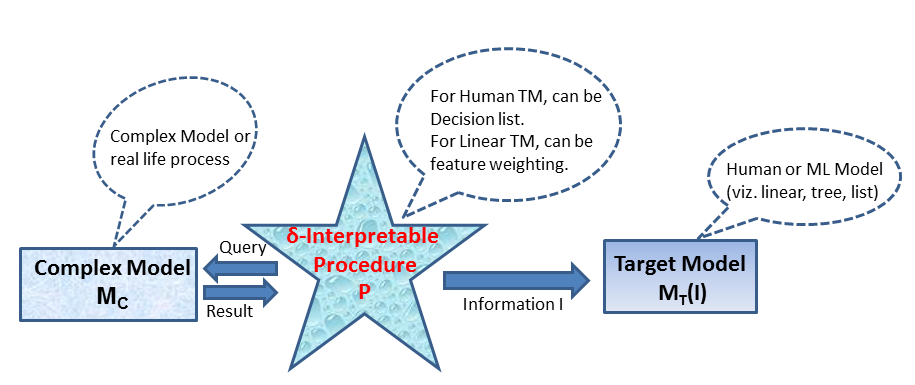}
     \caption{Above we depict what it means to be $\delta$-interpretable. Essentially, our procedure is $\delta$-interpretable if it improves the performance of TM by $\ge \delta$ fraction w.r.t. a target data distribution.}
\label{Intblck}
\end{figure*}


\section{Defining $\delta$-Interpretability}

Let us return to the opening question. Is interpretability simply sparsity, entropy, or something more general? An average person is said to remember no more than seven pieces of information at a time \cite{lisman1995storage}. Should that inform our notion of interpretability? Taking inspiration from the theory of computation \cite{toc} where a language is classified as regular, context free, or something else based on the strength of the machine (i.e. program) required to recognize it, we look to define interpretability along analogous lines.

Based on this discussion, we define interpretability relative to a target model (TM), i.e. $\delta$-interpretability. \emph{The target model in the most obvious setting would be a human, but it doesn't have to be}. It could be a linear model, a decision tree, or any simple model such that decisions of that model on specific cases/samples can be easily explainable (local model explanability using contributing factors) or the model itself naturally enhances understanding of an end-user in that domain. 
The TM in our running healthcare example \cite{chang2010} is a linear model where the features come from an ACG system mapping of IDC codes to only 32 Aggregated Diagnosis Groups (ADGs). This is a simple model for experts to interpret. \eat{In this simple TM, the mapping is based on only five clinical features, namely; duration of condition, severity of condition, diagnostic certainty, etiology of the condition and specialty care involvement, which surprisingly do not identify organ systems or disease.}

A procedure $P$ would qualify as being $\delta$-interpretable if it can somehow convey information to the TM that will lead to improving its performance (e.g., accuracy or AUC or reward) for the task at hand. 
%
%
\eat{Hence, the $\delta$-interpretable model has to transmit information in a way that is consumable by the TM.}
However, the information has to be transmitted in a way that is consumable by the TM.
What this means is that the hypothesis class of the TM remains the \textbf{same before and after} information transfer.\eat{For example, if the TM is a linear model our $\delta$-interpretable model can only tell it how to modify its feature weights or which features to consider.} For example, if the TM is a linear model the information can only tell it how to modify its feature weights or which features to consider.
In our healthcare example, the authors of \cite{chang2010} need a procedure to convey information from the complex 264-dimensional model to the simple linear 32-dimensional model. Any pairwise or higher order interactions would not be of use to this simple linear model.

We highlight three examples that showcase how our definitions in the latter part of this section capture different aspects of \textbf{human interpretability}. This is a testament to the generality of our definitions depicting its power to model varied situations. The main difference in these examples is the metric used to evaluate the performance of the (human) TM:
\begin{enumerate}
\item \textit{Improved Process:} In Section \ref{appl}, we report on an experiment in the advanced manufacturing domain where a rule list (CM) is shown to be $\delta$-interpretable relative to a semi-conductor engineer. Using insights from the rule list, the engineer was able to improve his manufacturing process (measured by wafer yield). Note that the engineer was not necessarily looking for local model explanations in this case.\eat{, something he wasn't able to do just by himself. This shows that he understood the rules learned by the rule list which was validated by his actions improving performance on the task at hand which was wafer yield in this case}.

\item \textit{Matching Belief:} Sometimes, the way to measure improvement in a human interpretable way is to check if the informed target model captures `existing human intuition/belief'. For example, in the setting where you want explanations for classifications of book reviews as positive or negative \cite{lime} (discussed in Section \ref{appl}), the metric that you are interested in improving is feature recall, which in this case is a set of phrases/words with positive or negative connotation that you would expect a good informed target model to pick up. Here, feature recall serves as a proxy for human confidence in a model. 
\eat{Hence, feature recall is the metric you want to enhance here, which can also be considered as a proxy for human confidence in a model in this example.}

\item \textit{New Insight:} In the olfaction experiment where we want to predict odor pleasantness, (discussed in Section \ref{olf}), an accurate random forest (CM) was used to transfer information to a lasso estimator (TM) nearly matching the performance of the CM. The top five features highlighted by lasso provided the experts with insight that enhanced their knowledge \eat{, which could be considered as a metric in this case,} motivating further lab experiments.
\end{enumerate}

Ideally, the performance of the TM should improve w.r.t. some target distribution. The target distribution could just be the underlying distribution, or it could be some reweighted form of it in case we are interested in some localities of the feature space more than others. For instance, in a standard supervised setting this could be generalization error (GE), but in situations where we want to focus on local model behavior the error would be w.r.t. the new reweighted distribution that focuses on a specific region. \emph{In other words, we allow for instance level interpretability as well as global interpretability and capturing of local behaviors that lie in between}. In this sense, one can capture aspects of local model explanability. The healthcare example focuses on mean absolute prediction error (MAPE) expressed as a percentage of the mean of the actual expenditure (Table 3 in \cite{chang2010}). Formally, we define $\delta$-interpretability as follows:

\par\vspace{0.3cm}
\fbox{\begin{minipage}{4.5 in} 
\begin{definition}\label{didef} $\delta$-interpretability: 
Given a target model $M_T$ belonging to a hypothesis class $\mathcal{H}$ and a target distribution $D_T$, a procedure $P$ is $\delta$-interpretable if it produces a model $ M^\prime_T$ in the same hypothesis class $\mathcal{H}$ satisfying the following inequality: $e_{M^\prime_T}\le \delta\cdot e_{M_T}$, where $e_{\mathcal{M}}$ is the expected error of $\mathcal{M}$ relative to some loss function on $D_T$.
\end{definition}
\end{minipage}}
\par\vspace{0.3cm}

The above definition is a general notion of interpretability that does not require the interpretable procedure to have access to a complex model. It may use a complex model (CM) and some other training data set, and statistics about the complex model's actions on that dataset to derive some useful information. However, it may very well act as an oracle conjuring up useful information that will improve the performance of the TM. When there is a CM, a more intuitive but special case of Definition \ref{didef} below defines $\delta$-interpretability based on the ability to transfer information from the CM to the TM using a procedure $P$ so as to improve TM's performance.\eat{
for a CM relative to a TM as being able to transfer information from the CM to the TM using a procedure $P_I$ so as to improve the TMs performance. 
} These concepts are depicted in figure \ref{Intblck}. 

\par\vspace{0.1cm}
\fbox{\begin{minipage}{4.5 in} 
\begin{definition} \label{didefa} CM-based $\delta$-interpretability: 
Given a target model $M_T$ belonging to a hypothesis class $\mathcal{H}$, a complex model $M_C$, and a target distribution $D_T$,
the procedure $P$ is $\delta$-interpretable relative to the model pair ($M_C$, $M_T$), if it derives information $I$ from $M_C$ and produces a model $M_T(I)\in\mathcal{H}$ satisfying the following inequality: $e_{M_T(I)}\le \delta\cdot e_{M_T}$, where $e_{\mathcal{M}}$ is the expected error of $\mathcal{M}$ relative to some loss function on $D_T$.
\end{definition}
\end{minipage}}
\par\vspace{0.3cm}
\eat{One may consider the more intuitive definition of $\delta$-interpretability when there is a CM.
}
We highlight two key ideas behind our definition - a) Interpretability is defined relative to a chosen target model that belongs to a simpler complexity class. In many applications, the target model class can be directly interpreted by a human/end user. This is the reason why the TM is allowed to change only within its hypothesis class. b) When people ask for an interpretation, there is an implicit quality requirement in that the interpretation should provide useful insight for a task at hand. We capture this relatedness of the interpretation to the task by requiring that the interpretable procedure improve the performance of the TM.
\eat{
Without such a requirement explanation may sometimes be arbitrary and useless making the concept of interpretation pointless. Consequently, the crux for any application in our setting is to come up with an interpretable procedure that can ideally improve the performance of a carefully chosen TM.
}


%
%

The closer $\delta$ is to 0 the more interpretable and useful the procedure. Note the error reduction is relative to the TM model itself, not relative to the complex model. An illustration of the above definition is seen in figure \ref{Intblck}. Here we want to interpret a complex process relative to a given TM and target distribution. The interaction with the complex process could simply be by observing inputs and outputs on some data set or could be through delving into the inner workings of the complex process. 
%
%
\eat{
\emph{In addition, it is imperative that $M_T(I)\in \mathcal{H}$ i.e. the information conveyed should be within the representational power of the TM.} 
}

We now clarify the use of the term Information $I$ in the definition. In a normal binary classification task, training label $y \in \{+1,-1\}$ can be considered to be a one bit information about the sample $x$, i.e., ``Which label is more likely given x?", whereas   the confidence score $p(y|x)$ holds richer information, i.e.,  ``How likely is the label y for the sample x?". From an information theoretic point of view, given x and only its training label $y$, there is still uncertainty about $p(y|x)$ prior to training. One possible candidate for information $I$ from a CM is something that could potentially reduce this uncertainty in the confidence score of a label $y$ on a sample $x$ prior to using the training procedure of the TM. Another possibility is that the target hypothesis class has a parameterized set of training algorithms, and an interpretable procedure can use the learned parameters of the complex model as a way to choose a specific training algorithm from the parameterized class.
%
%
\eat{can actually reduce the uncertainty of the $p(y(x)|x)$ in the interval $[1/2,1]$ prior to training the TM. However, the new $M_T(I)$ obtained is more useful if the training method can effectively use this potentially useful information. This is one precise and concrete sense in which, $I$ communicated by the interpretable method are indeed information bits that reduce uncertainty about the confidence score.}

\emph{The advantage of this definition is that the TM isn't tied to any specific entity such as a human and thus neither is our definition.} We can thus test the utility of our definition w.r.t. simpler models (viz. linear, decision lists, etc.), which could be perceived as lower bounds on human complexity. We see examples of this in the coming sections.
%
%
\eat{
%
%
Moreover, a direct consequence of our definition is that it naturally \emph{creates a partial ordering of interpretable procedures} relative to a TM and target distribution, which is in spirit similar to complexity classes for time or space of algorithms. For instance, if $\mathcal{R^+}$ denotes the non-negative real line $\delta_1$-interpretability $\Rightarrow$ $\delta_2$-interpretability, where $\delta_1 \le \delta_2$ $\forall \delta_1,~\delta_2\in \mathcal{R^+}$, but not the other way around. 
}

\section{Understanding Definition \ref{didef}}
In order to better understand our definition of $\delta$-interpretability, we ask the following question: how can one (i.e. an observer) validate the fact that information shared by one entity (or procedure) can be interpreted by another entity (or target model)? 
One may argue that the target model may communicate this directly to the observer, however this then assumes that the two are able to communicate. The most general setting is where no such assumption is made in which case the observer can only detect tangible transfer of information from the procedure to the target model by change in performance of the target model. 
The case where one wants to understand/interpret some concept is just a special case of this setting.
Even in this case, the only way to really know you understand the concept is to test oneself on a relevant task.
\eat{
The case where one wants to understand/interpret some concept is just a special case of this setting, where the observer and the target model can be considered to communicate perfectly or in essence are one and the same. 
}

\begin{figure}[h]
\centering
   \includegraphics[width=0.8\linewidth]{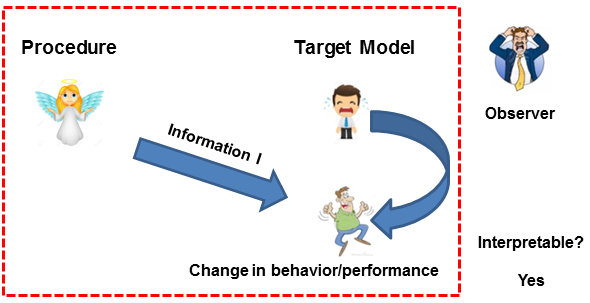}
     \caption{An intuitive justification of our definition.}
\label{defjust}
\end{figure}

We depict this concept in figure \ref{defjust}. An observer asks whether or not a given procedure is interpretable. A target model is selected and its initial state (viz., performance on a task) is observed, depicted by a weeping individual. The procedure conveys information $I$ to the target model, resulting in an updated target model whose state, depicted by a jovial individual, is observed. The observer thus has his answer regarding interpretability of the procedure. Hence, interpretability measures the change in behavior or performance (here being a change from sadness to happiness), and thus we truly need to define \emph{$\delta$-interpretability} rather than simply \emph{interpretability} in order to capture the impact of this change.

We can also think of this framework as the procedure being a teacher, the target model being a student, and the observer (i.e., the student's parent) wants to measure the quality of the teacher. The parent observes the student a priori on some task, and then measures the student on the same task after the teacher's lesson. The student's change in performance would dictate the teacher's ability to convey information in a manner that is interpretable to the student. Note that the students performance could become worse in which case $\delta > 1$, indicating that the teacher was interpretable but bad.

In fact, the evaluation of explanations in a recently conducted FICO explainability challenge ( https://community.fico.com/community/xml/pages/rules) closely matches our definition.

\section{Practical Definition of Interpretability}

We first extend our $\delta$-interpretability definition to the practical setting where we don't have the target distribution, but rather just samples. We then show how this new definition reduces to our original definition in the ideal setting where we have access to the target distribution.

\subsection{($\delta,\gamma$)-Interpretability: Performance and Robustness}

Our definition of $\delta$-interpretability just focuses on the performance of the TM. However, in most practical applications robustness is a key requirement. This could be quite crucial in fields like healthcare where small perturbations of the medical record does not produce drastic differences in methods of treatments. In fact, robustness to adversarial manipulations has been pointed out as a desirable element for any measure of interpretability \cite{lipton2016mythos}. The author in \cite{lipton2016mythos} gives the example of credit scoring models whose features can be manipulated by individuals by artificially requesting credit line increases which can be considered as adversarial manipulations. Given this we can extend our definition of  $\delta$-interpretability to account for robustness besides just performance. This also showcases the flexibility of our definition where orthogonal metrics, such as robustness in this case, can be added to better capture interpretability in diverse settings and applications.

So what really is a robust model? Intuitively, it is a notion where one expects the same (or similar) performance from the model when applied to ``nearby" inputs. In practice, this is many times done by perturbing the test set and then evaluating performance of the model \cite{carlini}. If the accuracies are comparable to the original test set then the model is deemed robust. Hence, this procedure can be viewed as creating alternate test sets on which we test the model. Thus, the procedures to create adversarial examples or perturbations can be said to induce a distribution $D_R$ from which we get these alternate test sets. \emph{The important underlying assumption here is that the newly created test samples are at least moderately likely w.r.t. target distribution.} Of course, in the case of non-uniform loss functions the test sets on whom the expected loss is low are uninteresting. This brings us to the question of when is it truly interesting to study robustness.

\emph{It seems that robustness is really only an issue when your test data on which you evaluate is incomplete i.e. it doesn't include all examples in the domain.} If you can test on all points in your domain, which could be finite, and are accurate on it then there is no need for robustness. That is why in a certain sense, low generalization error already captures robustness since the error is over the entire domain and it is impossible for your classifier to not be robust and have low GE if you could actually test on the entire domain. The problem is really only because of estimation on incomplete test sets \cite{kushtst}. \eat{Using conditional entropy definition for GE ie H(Y|X) as Karthik was suggesting doesnt solve the problem for incomplete test sets since we have seen time and again with these networks that they may give high confidences for the correct class on these test sets but are still not robust.} Given this we extend our definition of $\delta$-interpretability for practical scenarios.

\par\vspace{0.25cm}
\fbox{\begin{minipage}{4.5 in}
\begin{definition}($\delta,\gamma$)-interpretability:\label{didefp} Given a target model $M_T$ belonging to a hypothesis class $\mathcal{H}$, a sample $S_T$ from the target distribution $D_T$, a sample $S_R$ from a distribution $D_R$, a procedure $P$ is ($\delta, \gamma$)-interpretable relative to $(D_T\sim) S_T$ and $(D_R\sim) S_R$ if it produces a model $M^\prime_T\in\mathcal{H}$ satisfying the following inequalities:
\begin{itemize}
\item $\hat{e}^{S_T}_{M^\prime_T}\le \delta\cdot \hat{e}^{S_T}_{M_T}$ (performance)
\item $\hat{e}^{S_R}_{M^\prime_T}-\hat{e}^{S_T}_{M^\prime_T}\le \gamma\cdot (\hat{e}^{S_R}_{M_T}-\hat{e}^{S_T}_{M_T})$ (robustness)
\end{itemize}
where $\hat{e}^{\mathcal{S}}_{\mathcal{M}}$ is the empirical error of $\mathcal{M}$ relative to some loss function.
\end{definition}
\end{minipage}}
\par\vspace{0.25cm}

The first term above is analogous to the one in Definition \ref{didef}. The second term captures robustness and asks how representative is the test error of $M^\prime_T$ w.r.t. its error on other high probability samples when compared with the performance of $M_T$ on the same test and robust sets. This can be viewed as an orthogonal metric to evaluate interpretable procedures in the practical setting. This definition could also be adapted to a more intuitive but restrictive definition analogous to Definition \ref{didefa}.

\subsection{Examples of $D_R$}

The distribution $D_R$ is the alternate distribution to $D_T$ that we wish to test our model on. We do not have to explicitly define $D_R$ as it could be an induced distribution based on some process that generates samples. Below are some examples of processes that can induce $D_R$.

\begin{itemize}
\item \emph{Adversarial attacks:} Robustness of Neural Networks is an active research area. Adversarial attacks \cite{carlini} is one of the primary ways in which to test the robustness of these models. The attacks perturb test samples so that they are virtually indistinguishable to a human but fool a deep neural network. Thus, the attacks can be viewed as inducing a distribution $D_R$ where the perturbed test samples can be seen to represent $S_R$. The distillation example in the next section uses an adversarial attack to compute $\gamma$ for a deep neural network.
\item \emph{Domain Shift:} In many applications, the data that a model is trained and tested on may have a different distribution than the data that a model sees on deployment \cite{lcf}. The distribution of the data in the deployed setting can be seen to represent $D_R$. The prototype selection example in the next section tests the mmd-critic method on a skewed digit distribution, in addition to the original one, representing $D_R$.
\item \emph{Random Noise:} To test robustness of methods many times slight random perturbations are added to samples or a small fraction of labels are flipped. These processes of randomly perturbing the data can again be perceived as inducing a distribution $D_R$ over samples $S_R$ that are generated from them. In the synthetic experiment in section \ref{simexp} we perform random label flips for a small fraction of instances (5\%), which induces $D_R$ and we compute $\gamma$ for the linear TM.
\end{itemize}

\subsection{Reduction to Definition \ref{didef}}

\eat{\emph{An (ideal) adversarial example is not just one which a model predicts incorrectly, but rather it must satisfy also the additional requirement of being a highly probable sample from the target distribution.} Without the second condition even highly unlikely outliers would be adversarial examples. But in practice this is not what people usually mean, when one talks about adversarial examples.}

Sometimes, some models can be exhaustively trained with large number of samples from a target distribution that produces \textit{all} the realistic samples that the model could be tested on. An example would be a huge image corpus consisting of a few hundred million images (including all perturbations of images that are considered meaningful). Given this, ideally, we should choose $D_R=D_T$ so that we test the model mainly on important examples. If we could do this and test on the entire domain our Definition \ref{didefp} would reduce to Definition \ref{didef} as seen in the following proposition.

\begin{proposition}
In the ideal setting, where we know $D_T$, we could set $D_R=D_T$ and compute the true errors, ($\delta,\gamma$)-interpretability would reduce to $\delta$-interpretability.
\end{proposition}
\begin{proof}
Since $D_R=D_T$, by taking expectations we get for the first condition: 
\begin{equation*}
\begin{split}
E[\hat{e}^{S_T}_{M^\prime_T} - \delta\hat{e}^{S_T}_{M_T}]&\le 0\\
e_{M^\prime_T}-\delta\cdot e_{M_T}&\le 0
\end{split}
\end{equation*}

For the second equation we get: 
\begin{equation*}
\begin{split}
E[\hat{e}^{S_R}_{M^\prime_T} - \hat{e}^{S_T}_{M^\prime_T}-\gamma \hat{e}^{S_R}_{M_T}+\gamma\hat{e}^{S_T}_{M_T}]&\le 0\\
e_{M^\prime_T}-e_{M^\prime_T}-\gamma e_{M_T}+\gamma e_{M_T}&\le 0\\
0&\le 0.
\end{split}
\end{equation*}
\end{proof}

The second condition vanishes and the first condition is just the definition of $\delta$-interpretability. Our extended definition is thus consistent with Definition \ref{didef} where we have access to the target distribution.

\vspace{.5cm}
\noindent\textbf{Remark:} Model evaluation sometimes requires us to use multiple training and test sets, such as when doing cross-validation. In such cases, we have multiple target models $M^i_T$ trained on independent data sets, and multiple independent test samples $S^i_T$ (indexed by $i=\{1,\ldots,K\}$). The 
empirical error above can be defined as $(\sum_{i=1}^K{\hat{e}_{M^i_T}^{S^i_T}})/K$. Since $S^i_T$, as well as the training sets, are sampled from the same target distribution $D_T$, the reduction to Definition \ref{didef} would still apply to this average error, since $E[\hat{e}_{M^h_T}^{S^i_T}] = E[\hat{e}_{M^j_T}^{S^k_T}]$ $\forall$ $h,i,j,k$.

\begin{table*}[t]
\begin{center}
  \begin{tabular}{|c|c|c|c|c|c|c|}
    \hline
Interpretable & TM & $\delta$ & $\gamma$ & $D_R$ & Dataset ($S_T$) & Performance\\
Procedure &&&&&& Metric\\
\hline
\hline
EDC Selection & OLS & 0.925 & 0 & Identity & Medical Claims & MAPE\\
(Chang and Weiner 2010)&&&&&&\\
\hline
Distillation & DNN & 1.27  & 0.8 & $L_2$ attack & MNIST & Classification\\
(Carlini and Wagner 2017)&&&&&&error\\
\hline
MMD-critic &  & 0.92 & 0.98 &  &  & \\
K-medoids &  & 0.73 & 0.96 &  &  & \\
ProtoDash & NPC & 0.57 & 0.97 & Skewed & MNIST & Classification\\
ProtoGreedy &  & 0.57 & 0.97 &  &  & error\\
P-lasso &  & 1.15 & 0.98 &  &  & \\
\hline
LIME & SLR & 0.1 & 0 & Identity & Books & Feature\\
(Ribeiro et. al. 2016)&&&&&&Recall\\
\hline
Interpretable MDP & Static & 0.579 & 0 & Identity & TUI Travel & Conversion \\
& & & & & Products & Rate \\
(Petrik and Luss 2016)&&&&&&\\
 \hline
Rule Lists & Human & 0.95 & 0 & Identity & Manufacturing & Yield\\
\hline
Tree Extraction & Human & 0.4 & 0 & Identity & Health care & Fidelity\\
(Bastani et. al. 2017)&&&&&&\\
\hline
Decision Set & Human & 0.06 & 0 & Identity & Medical user & Classification\\
Decision List & Human & 0.36 & 0 & Identity & study & error\\
(Lakkaraju et. al. 2016)&&&&&&\\
\hline
\end{tabular}
\end{center}
  \caption{Above we see how our framework can be used to characterize interpretability of methods across applications.}
\label{inttbl}
\end{table*}

\section{Application to Existing Interpretable Methods}
\label{appl}

We now look at some current methods and how they fit into our framework.\\

\noindent\textbf{EDC Selection:} The running healthcare example of \cite{chang2010} considers a complex model based on 264 EDC (Expanded Diagnosis Codes) features and a simpler linear model based on 32 ACG (Adjusted Clinical Groups) features, and both models also include the same demographic variables. The complex model has a MAPE of 96.52\% while the linear model has a MAPE of 103.64\%. The authors in \cite{chang2010} attempt to improve the TM's performance by including some EDC features. They develop a stepwise process for generating selected EDCs based on significance testing and broad applicability to various types of expenditures (\cite{chang2010}, Additional File 1). This stepwise process can be viewed as a $(\delta, \gamma)$-interpretable procedure that provides information in the form of 19 EDC variables which, when added to the TM, improve the performance from 103.64\% to 95.83\% and is thus (0.925, 0)-interpretable, since $0.925=\frac{95.83}{103.64}$ and there is no robustness test so $D_R$ is identity which means it is same as $D_T$ making $\gamma=0$. Note the significance since, given the large population sizes and high mean annual healthcare costs per individual, even small improvements in accuracy can have high monetary impact.

One may argue here as to how the hypothesis class or complexity is maintained by the interpretable procedure as we are adding variables. Although it is not maintained in the standard VC dimension \cite{vapnik} sense, the complexity here is what the experts consider as interpretable and for them adding these variables, which are not in the 100s, is interpretable or can be understandably consumed by them.\\

\noindent\textbf{Distillation:} Distillation is a method to train a possibly smaller neural network using softmax scores of a larger pretrained neural network on a training dataset after a suitable temperature scaling of the final softmax layer \cite{distill}. A special version of this is called \textit{defensive distillation} \cite{carlini} if the sizes of both neural nets remain the same while a very high temperature is used for the softmax score scaling. The purpose of defensive distillation  is to add more robustness to the model. In the case of distillation (defensive or otherwise), if you consider a DNN to be a TM then you can view defensive distillation as a $(\delta, \gamma)$-interpretable procedure between two neural networks. We compute $\delta$ and $\gamma$ from results presented in \cite{carlini} on the MNIST dataset for a state-of-the-art deep network, where the authors adversarially perturb the test instances, such that the distortion introduced is  `imperceptible' to the human eye, and try to check the robustness of any neural network. The $\delta$ is computed using accuracies of the original and distilled networks on the unperturbed test dataset. $\gamma$ is computed based on the accuracies on the adversarially perturbed test dataset for both the original and the distilled network. We see here that defensive distillation makes the DNN slightly more robust at the expense of it being a little less accurate. \\

\noindent\textbf{Prototype Selection:} We compare 5 prototype selection algorithms -- MMD-critic \cite{l2c}, K-medoids \cite{sproto}, P-Lasso \cite{classo}, ProtoDash and ProtoGreedy \cite{proto} -- in our framework. 
The TM was a nearest prototype classifier (NPC) \cite{l2c} that was initialized with 200 random prototypes which it used to create the initial classifications. We implemented and ran these prototype selection algorithms on randomly sampled MNIST training sets of size 1500 where the number of prototypes was set to 200. The test sets were 5420 in size which is the size of the least frequent digit. We had a representative test set and then 10 highly skewed test sets where each contained only a single digit. This setup is similar to the one in \cite{proto}. The representative test set was used to estimate $\delta$ and the 10 skewed test sets were used to compute $\gamma$.

The accuracy of the NPC using random prototypes was 74.1\% on the representative set. The corresponding accuracies using MMD-critic, K-medoids, P-Lasso, ProtoDash and ProtoGreedy were 76.3\%, 79.2\%, 70.1\%, 85.2\% and 85.2\% respectively. These accuracies were similar on the skewed datasets. We see in Table \ref{inttbl} the $\delta$s and $\gamma$s computed based on these numbers. We observe that although majority of the methods are interpretable w.r.t. the NPC TM, ProtoDash and ProtoGreedy are the most. The robustness of all these methods is marginally better than the (initial) baseline TM.\\

\noindent\textbf{LIME:} We consider the experiment in \cite{lime} where they use sparse logistic regression (SLR) as a target model to classify a review as positive or negative on the Books dataset. The SLR model is built based on an already trained complex model that they want to interprete. They also train an SLR model based on random feature selection. Their main objective here is to see if their interpretable procedure is superior to other approaches in terms of selecting the true important features. Hence, the performance metric here is the fractional overlap between the features highlighted by an interpretable method and the true set of important features that have been indicated by human experts. We observe that their performance in selecting the important features (error=7.9\%) is significantly better than random feature selection (error=82.6\%) which can be quantified by our approach based on the corresponding errors as $\delta=0.1$. 
The other experiments can also be characterized in similar fashion. In cases where only explanations are provided with no explicit metric one can view the experts confidence in the method as a metric which good explanations will enhance.\\

\noindent\textbf{Interpretable MDP:} The authors used a constrained MDP formulation \cite{imdp} to derive a product-to-product recommendation policy for the European tour operator TUI. The goal was to generate buyer conversions and to improve a simple product-to-product policy based on static pictures of the website and what products are currently looked at. The constrained MDP results in a policy that is just as simple but greatly improves the conversion rate which is averaged over 10 simulations where customer behavior followed a mixed logit customer choice model with parameters fit to TUI data. The mean normalized conversion rate increased from 0.3377 to 0.6167. This leads to a $\delta$ value of 0.579.\\

\noindent\textbf{Rule Lists:}
The authors built a rule list on a semi-conductor manufacturing dataset \cite{jmlr14Amit} of size 8926. In this data, a single datapoint is a wafer, which is a group of chips, and measurements, that corresponds to 1000s of input features (temperatures, pressures, gas flows, etc.), made on this wafer throughout its production. The goal was to provide the engineer some insight into what, if anything, was plaguing his process so that he can improve performance. The authors built a rule list \cite{twl} of size at most 4 which we showed to the engineer. The engineer figured out that there was an issue with some gas flows which he then fixed. This resulted in 1\% more of his wafers ending up within specification. In other words, his yield increased from 80\% to 81\%, which corresponds to a $\delta$ value of 0.95. This is significant since even a small increase in yield corresponds to billions of dollars in savings.\\

\noindent\textbf{Tree Extraction:} In \cite{bastani2017interpreting}, the authors acquired patient records from a leading electronic medical records company. The data they used had 578 patients from which they obtained 382 features. The features contained demographic information as well as medical history of each patient. The goal was to predict if a patient had type II diabetes in his most recent visit. The authors conducted experiments using standard decision trees and their approach of decision tree extraction. The complex model was a random forest. They observed that their method produced (interpretable) decision trees closest in performance to the random forest. The fidelity of the standard decision tree was 0.85, while that of the extracted tree was 0.94. From these we can compute $\delta=\frac{0.06}{0.15}$ which is reported in table \ref{inttbl}. Similar qualitative gains were seen with human experts.\\

\noindent\textbf{Decision Sets and Decision Lists:} In \cite{Lakk}, the authors conducted an online user study with 47 students based on a health care dataset. The goal was if based on the symptoms the students could correctly estimate if the patient had a certain disease. There we 10 questions and each question had a binary outcome where student had to answer "true" or "false" given the potential disease. Without any information and based on random guessing the baseline performance would be 50\% in expectation. The authors used their method of decision sets and another method decision lists to convey rules connecting symptoms to diseases. They found that their decision set improved the accuracy of the students to 97\%, while the decision list led to an accuracy of 82\%. Based on this we could compute $\delta$s that are reported in table \ref{inttbl}. This shows that the decision set was significantly more interpretable for the conducted user study.

\section{Framework Generalizability}
\label{assump}

It seems that our definition of $\delta$-interpretability requires a predefined goal/task. While (semi-)supervised settings have a well-defined target, we discuss other applicable settings.


In unsupervised settings, although we do not have an explicit target, there are quantitative measures \cite{charubook} such as Silhouette or mutual information that are used to evaluate clustering quality. Such measures which people use to evaluate quality would serve as the target loss that the $\delta$-interpretable procedure would aim to improve upon. The target distribution in this case would just be the instances in the dataset. If a generative process is assumed for the data, then that would dictate the target distribution.

In other settings such as reinforcement learning (RL) \cite{reinf}, the $\delta$-interpretable procedure would try to increase the expected discounted reward of the agent by directing it into more lucrative portions of the state space. In inverse RL on the other hand, it would assist in learning a more accurate reward function based on the observed behaviors of an intelligent entity. The methodology could also be used to test interpretable models on how well they convey the causal structure \cite{pearl} to the TM by evaluating the TMs performance on counterfactuals before and after the information has been conveyed.

\section{Candidate (Model Agnostic) $\delta$-Interpretable Procedures}

We now provide theoretically grounded $\delta$-interpretable strategies that use the CMs confidence scores to weight the training data that the TM trains on. Although the procedures are described for binary classification they straightforwardly extend to multiclass settings. We then illustrate how these procedures improve the performance of simple TMs on 5 real data sets: olfaction data, MNIST, CIFAR-10, a manufacturing dataset and a home loan dataset \cite{FICO}. A 2-dimensional synthetic example further offers a visual understanding of $\delta$-interpretability.

\subsection{Derivation of $\delta$-interpretable Procedures}

We now derive $\delta$-interpretable procedures based on different classification frameworks as well as error bounds for these procedures.

\subsubsection*{Two popular Frameworks for Classification}
We discuss two popular frameworks for binary classifiers based on their training methods. 

\noindent\textit{1. Expected Risk Minimization Models (ERM):}
Suppose the target model is optimized according to empirical risk minimization on $m$ training samples using the risk function $r(y,x,\mathbf{\theta})$ given by:
$\min \limits_{\theta} \frac{1}{m} \sum_{i=1}^m r(y_i,x_i,\mathbf{\theta}) $. Let us assume that $0 \leq r(\cdot) \leq 1$. The classification rule for a new sample $x$ is $\argmin_{y \in \{+1,-1\}} r(y,x,\mathbf{\theta})$.

\noindent\textit{2. Maximum Likelihood Estimation Models (MLE):}
In this case, the binary classifier is specified directly by $p_{\mathrm{TM}}(y|x; \mathbf{\theta})$. Given $m$ training samples $(y_i,x_i)$, the following likelihood optimization is performed:
\[ \min \limits_{\theta} \frac{1}{m} \sum_{i=1}^m - \log p_{\mathrm{TM}}(y_i|x_i; \mathbf{\theta}).\] The classification rule for a new sample $x$ in this case is $\argmax_{y \in \{+1,-1\}} p_{\mathrm{TM}}(y|x)$.

 Let the shorthand notation $r_1(x)$ denote $r(+1,x,\mathbf{\theta})$ while $r_2(x)$ denote $r(-1,x,\mathbf{\theta})$. Let $y'(x)= \argmax \limits_{y} p_{\mathrm{CM}}(y|x)$, $r_1^{y'}(x)=r(y'(x),x,\mathbf{\theta})$, and $r_2^{y'}(x)=r(-y'(x),x,\mathbf{\theta})$. 

\subsubsection*{Generalization Error bounds}
Let us assume that the complex model $\mathrm{CM}$ is highly accurate. Hence, we assume that the data is generated according to the distribution ${\cal D}_{\mathrm{CM}} = p(x) p_{\mathrm{CM}} (y|x)$. The error obtained by applying the TM on the data in the MLE case is given by: $\mathbb{E}_{{\cal D}_{\mathrm{CM}}}[ \mathbf{1}_{p_{\mathrm{TM}}(y|x;\mathbf{\theta})<=1/2} ]$. For the ERM case, it is given by: $\mathbb{E}_{{\cal D}_{\mathrm{CM}}}[ \mathbf{1}_{r(y,x,\mathbf{\theta})>r(-y,x,\mathbf{\theta})}]$. Theorem \ref{mainthm} below (proof in appendix) bounds the squared error for both ERM and MLE, along with a reformulation of ERM error. The first term in the bounds is either the weighted ERM or the weighted MLE problem. The second term in ERM (a) and MLE penalizes the margin of the TM classifier, while the second term in ERM (b) is the quantization error incurred from converting confidence scores to hard classifications.

\begin{theorem} The error bounds for the ERM and MLE cases are as follows:\\
  \textbf{ERM case (a):}\\  $\mathbb{E}^2_{{\cal D}_{\mathrm{CM}}}[ \mathbf{1}_{r(y,x,\mathbf{\theta})>r(-y,x,\mathbf{\theta})} ]\leq \mathbb{E}_{p(x)} \left[c\cdot p_{\mathrm{CM}}(+1|x) r_1(x) + c\cdot p_{\mathrm{CM}}(-1|x) r_2(x) \right]$  \\
  $ +\mathbb{E}_{p(x)}\left[\log (1 +e^{-c \lvert r_1(x) - r_2(x) \rvert}) + 2e^{-2c\lvert r_1(x) - r_2(x) \rvert}\right] $  \\\\
  \textbf{ERM case (b):} \\
$\mathbb{E}_{{\cal D}_{\mathrm{CM}}}[\mathbf{1}_{r(y,x,\theta)>r(-y,x,\theta)}]= \mathbb{E}_{p(x)} \left[ 2\left| \frac{1}{2}-p_{\mathrm{CM}}(y'(x)|x) \right| \cdot\mathbf{1}_{r_1^{y'}(x)> r_2^{y'}(x)}\right.$ \\
$+ \left.\frac{1}{2}- \left| \frac{1}{2}-p_{\mathrm{CM}}(y'(x)|x) \right| \right]$\\\\
\textbf{MLE case:}\\
  $\mathbb{E}^2_{{\cal D}_{\mathrm{CM}}}[ \mathbf{1}_{p_{\mathrm{TM}}(y|x;\mathbf{\theta})<=1/2} ]\leq  \mathbb{E}_{p(x)} \left[ -(p_{\mathrm{CM}}(+1|x)) \log( p_{\mathrm{TM}}(+1|x;\mathbf{\theta} )) \right.$\\  
  $\left. -p_{\mathrm{CM}}(-1|x)\log( p_{\mathrm{TM}}(-1|x;\mathbf{\theta} )) + 2e^{-2\lvert \log p_{\mathrm{TM}}(-1|x;\mathbf{\theta}) - \log p_{\mathrm{TM}}(+1|x;\mathbf{\theta})  \rvert} \right]$
\label{mainthm}
\end{theorem}

 We next provide three $\delta$-interpretable methods.\eat{ Two are well-suited for ERM models while the last is well-suited for a maximum likelihood estimator (MLE) model.} The methods are motivated by the above theorem that offers bounds and reformulations of ERM and MLE generalization error. The function $f(\cdot)$ below is non-increasing on the domain $(0,\infty)$. In practice, $f(\cdot)$ can be any function that optimizes the margin of the given target model. The theorem above specifies candidate functions $f(\cdot)$ for both margin based $\delta$-interpretable methods.\\
\eat{
The three results of this theorem motivate three $\delta$-interpretable procedures, two for the ERM case and one for the MLE case. Each procedure relies on obtaining the confidence scores $p_{\mathrm{CM}}(y|x_i)$ for each sample $x_i$ of the training data. If a classifier does not output confidence scores we can obtain them using ideas from \cite{zadrozny2002transforming}. The three procedures are to respectively solve the following three problems:
}

\noindent\textbf{ERM case (a):}\\$ \min \limits_{\theta} \frac{1}{m} \left[ \sum_{i=1}^m \sum_{y\in \{+1,-1\}} c\cdot p_{\mathrm{CM}}(y|x_i)r(y,x_i,\mathbf{\theta})+ f(c|r_1(x_i)-r_2(x_i)|) \right]$ \\

\noindent\textbf{ERM case (b):}\\$   \min \limits_{\theta} \frac{1}{m} \left[ \sum_{i=1}^m \left|\frac{1}{2}- p_{\mathrm{CM}}(y'(x_i)|x_i) \right| r(y'(x_i),x_i,\mathbf{\theta}) \right]$\\

\noindent\textbf{MLE case:}\\$\min \limits_{\theta} \frac{1}{m} \left[ \sum_{i=1}^m \sum_{y\in \{+1,-1\}} -p_{CM}(y|x_i) \log p_{\mathrm{TM}}(y|x_i; \mathbf{\theta}) +f(|\log p_{\mathrm{TM}}(+1|x_i; \mathbf{\theta}) - \log p_{\mathrm{TM}}(-1|x_i; \mathbf{\theta}) |) \right]$\\

Note that there is a hyper-parameter $c>0$ that must be tuned for ERM case (a). The two ERM cases optimize the risk functionals while the MLE case directly optimizes the confidence of the TM.

\noindent\textbf{Remark:} In all the above methods, if we ignore the terms corresponding to margins (the $f$ function), primarily it involves weighing sample loss function by $p_{\mathrm{CM}}(y|x)$ or $p_{\mathrm{CM}}(y'|x)$ where $y$ is the label from the training data and $y'$ corresponds to the prediction of the complex model. The basic intuition is that if the complex model is unsure of the label (either the predicted one or the training label), this means $p_{\mathrm{CM}}$ is very small indicating high uncertainty. Therefore, the loss function of this sample is down weighted since this is difficult to classify even for the complex model. From the perspective of the classification boundary, $p_{\mathrm{CM}}$ being small indicates samples being close to the decision boundary and hence has a higher probability of cross over. We down weight these samples to force the target models to focus on easier to classify examples in the bulk of the decision boundary. $f$ functions just makes sure that the confidence in the predictions of the target model is as high as possible - a type of margin condition (as in SVMs). Imposing custom margin conditions can be difficult to implement with arbitrary training algorithms. Hence, one can drop the regularization due to $f$ terms and focus on the weighting just based on the respective functions of the confidence scores.

\eat{
Note that the ERM case (a) and MLE case suggest duplicating the training set for each label. However, the implementations of popular classification algorithms in standard packages are many times confused by this duplication. Hence, in practice, we suggest passing just one copy containing the predicted label with the corresponding weight when training the TM.}

\subsection{Evaluation on Simulated Data}
\label{simexp}
We next illustrate how the above $\delta$-interpretable procedures can be used to improve a simple target model. Simulated data, a target model $M_T$ and improved target model $M_T(I)$ are shown in figure \ref{fig:simulation} (left). Two classes (green circles and red diamonds) are uniformly sampled (1000 instances) from above and below the blue curve, so a linear model is clearly suboptimal. Label noise (5\%) was added primarily in the upper left corner.

We show here that a k-nearest neighbors (knn) classifier is $(\delta,\gamma)$-interpretable relative to a linear model w.r.t. the MLE interpretable procedure. The linear target model $M_T$ (dashed black line) is obtained by a logistic regression and achieves an accuracy of 0.766.  A k-nearest neighbors classifier achieving accuracy 0.856 was used to generate confidence scores via \cite{zadrozny2002transforming}, and solving the above problem for the MLE case results in the improved model $M_T(I)$ (solid red line) which achieves 0.782 accuracy. For the robustness test, 10\% of the labels were flipped, which resulted in $M_T$ accuracy falling to 0.71, while $M_T(I)$ accuracy fell to 0.722. Based on our definitions this implies that our MLE procedure is (0.931,1.071)-interpretable in this case.

Figure \ref{fig:simulation} (right) zooms in on a section of the left figure, exhibiting the benefit of the procedure: Several green circles misclassified by the target model $M_T$ are classified correctly by the reweighted model $M_T(I)$. 

\begin{figure}[t] \begin{center}
  \begin{tabular} {cc}
     \includegraphics[width=0.5\textwidth]{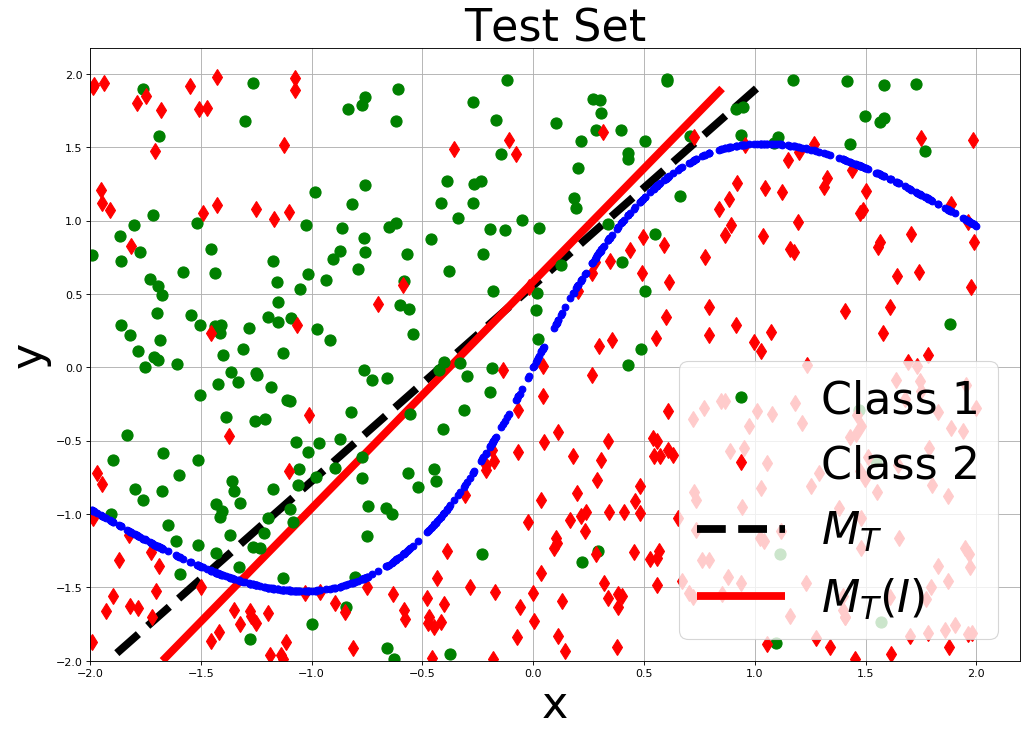}&
     \includegraphics[width=0.5\textwidth]{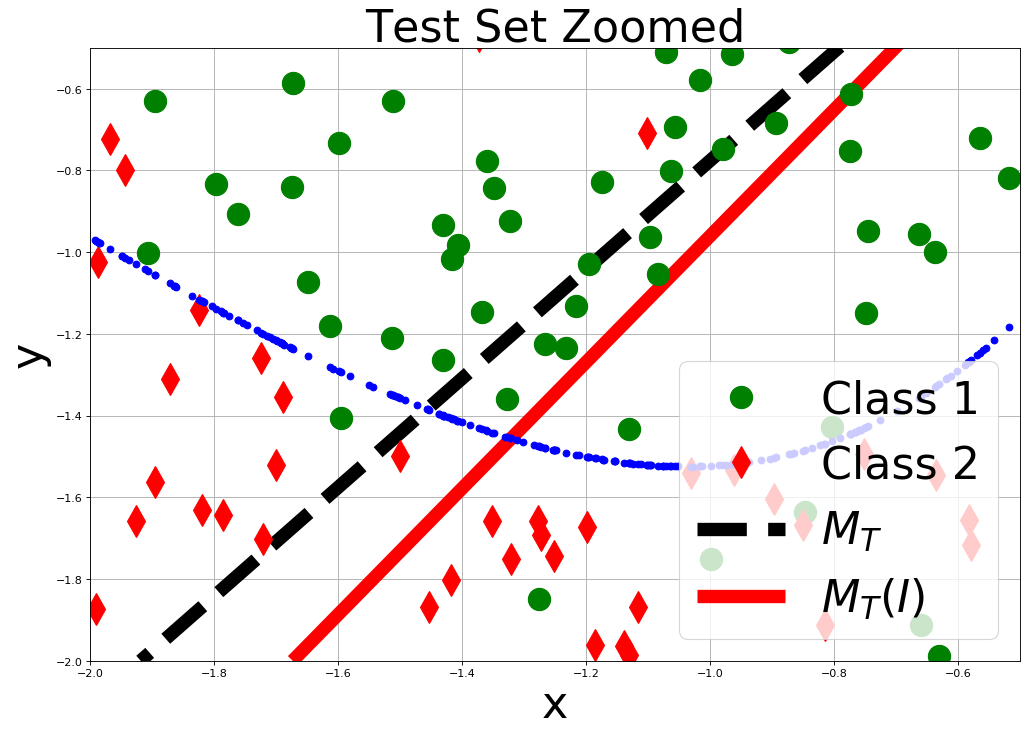}
  \end{tabular}
\caption{Illustration of $\delta$-interpretability on simulated data.} \label{fig:simulation}
\end{center} \end{figure}

\begin{figure}[t]
\centering
\includegraphics[width=0.6\textwidth]{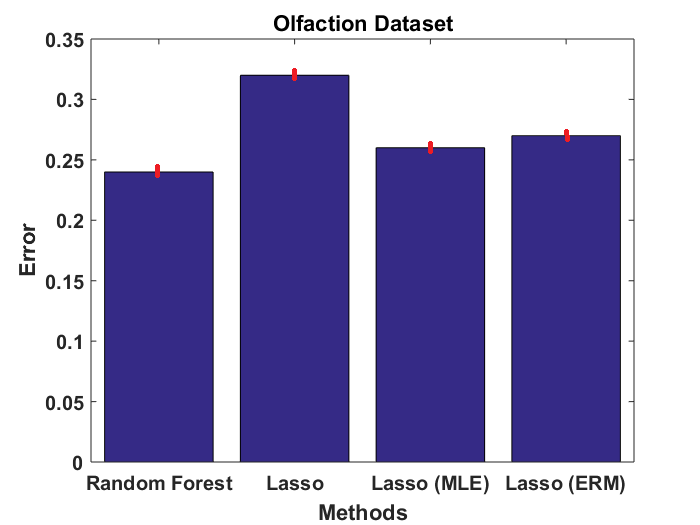}
\caption{Above we witness our MLE and ERM procedures being $\delta$-interpretable relative to a simple lasso model on (real) olfaction data.}
\label{fig:olfaction}
\end{figure}
\subsection{Evaluation on (Real) Olfaction Data}
\label{olf}
We evaluated our strategies on a recent publicly available olfaction dataset \cite{olfs} which has hundreds of molecules and thousands of chemoinformatic features along with qualitative perceptions averaged across almost 50 individuals. We chose \emph{Pleasantness}, which was one of the major percepts in this dataset as the target. The scale for this percept went from 0 to 100, where 0 meant that the odor was highly unpleasant while 100 meant that it was extremely pleasant. Hence, odors which were at 50 could be considered as neutral odors. For our binary classification setting we thus created two classes where class 1 was all odors with pleasantness $<50$ while class 2 was all odors with pleasantness $> 50$. We performed 10-fold cross-validation to obtain the results (mean with 95\% confidence interval) shown in figure \ref{fig:olfaction}.

We used a random forest (RF) model as our CM which had a test error of 0.24. Our TM was lasso which had a test error of 0.32. Using our MLE interpretable procedure the error of lasso dropped to 0.26. While using our ERM (b) strategy the error dropped to 0.27. This is depicted in figure \ref{fig:olfaction}. Hence, our MLE procedure was $(0.81, 0)$-interpretable, while our ERM (b) procedure was $(0.84, 0)$-interpretable. This illustrates a manner in which two interpretable procedures can be compared quantitatively, where in this example, the MLE procedure would be preferred.\\

\noindent\textbf{Human interpretable features highlighted:} 
An additional benefit of having a high performing (simple) lasso model is that important input features and their contribution can be readily highlighted to humans. The top features for our MLE model were R8v+, JGI7, R4p+, GGI9 and R6m+. When we shared this finding with experts they informed us that these features essentially characterized the shape and geometry of the molecule along with the global charge transfer characteristics within the molecule. Based on such insights, further studies are being carried out for this and other perceptions \cite{nlpsmell}.

\begin{figure}[t] \begin{flushleft}
  \begin{tabular} {cc}
     \includegraphics[width=0.45\textwidth]{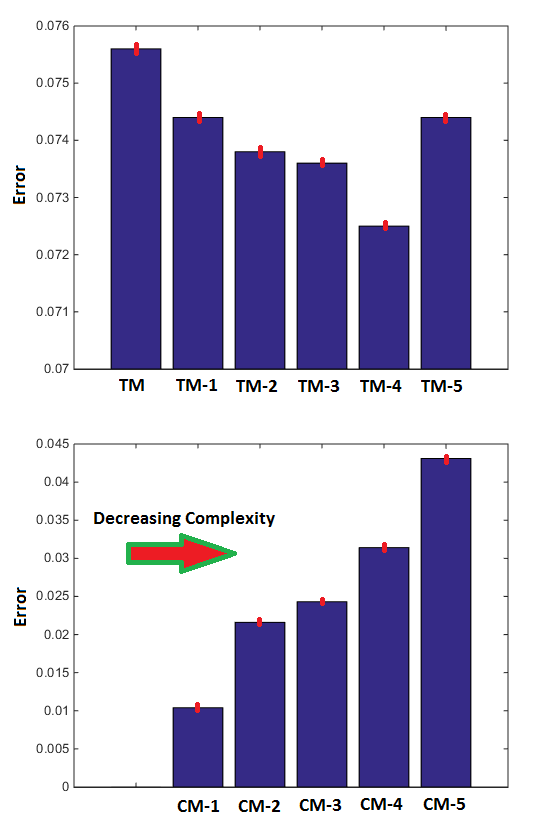}&
     \includegraphics[width=0.6\textwidth]{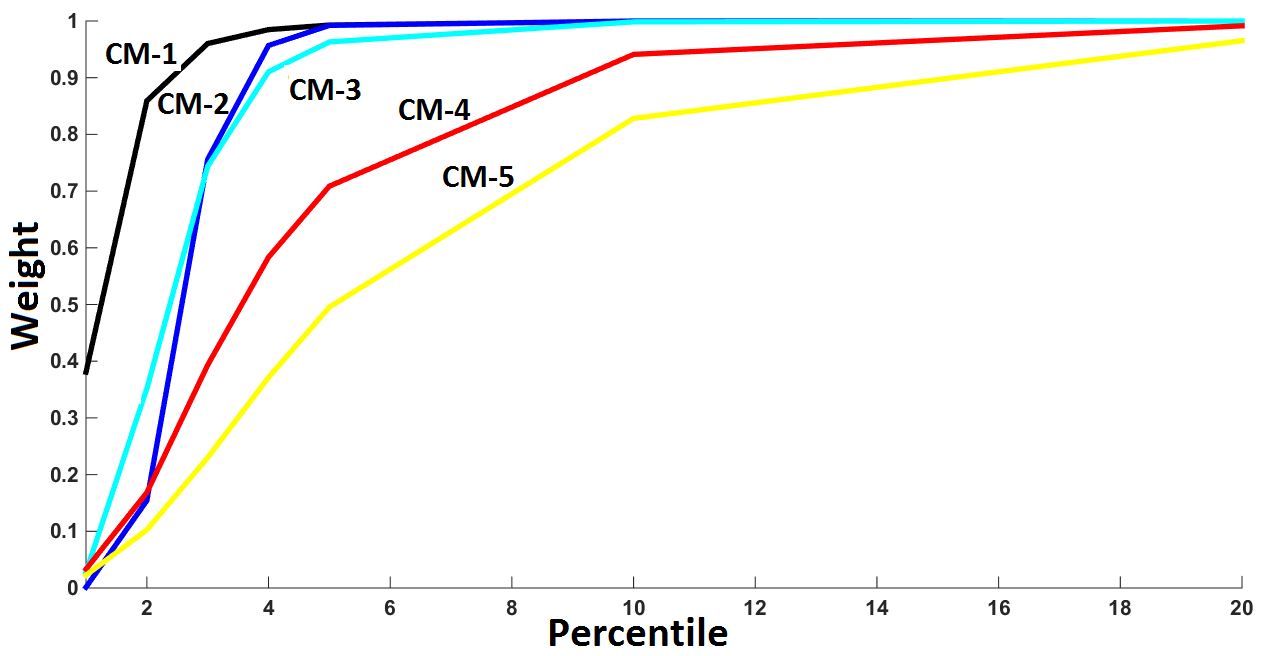}
  \end{tabular}
\caption{Above we see results on the MNIST dataset. In the above \emph{left} figure we see that the complex model CM-4 which is of intermediate complexity and performance produces the greatest in improvement in our TM (denoted by TM-4) given our interpretable strategy. The \emph{right} figure depicts the distribution of confidence scores for each CM used to weight the instances for training the TM.} \label{fig:mnist}
\end{flushleft} \end{figure}

\subsection{Evaluation on MNIST}

We now test the hypothesis if having a better complex model also implies that the $\delta$ will be lower for a given TM.

\subsubsection{Setup} 

We build 5 complex models of decreasing complexity. The complexity could be characterized by the number of parameters used to train the models. The most complex model CM-1 we use is given as a candidate to use on MNIST in keras \cite{kerasmnistCNN} which has around 1.2 million parameters and is a 4 layer network. CM-2 \cite{kerasmnistMLP} is of slightly lower complexity with around 670K parameters and is a 3 layer network. CM-3, CM-4 and CM-5 are 2 layer networks with 512, 64 and 32 rectified linear units respectively and a 10-way softmax layer. They have approximately 400K, 50K and 25K parameters. Our TM is a single layer network with just a softmax layer and has close to 8K parameters.

We split the MNIST training set randomly into two equal parts train1 and train2. We train our TM on train2. We then train the CMs on train1 and make predictions on train2. Using our MLE interpretable strategy we derive corresponding weights for instances in train2. We then train our TM using the corresponding weighted examples and obtain 5 corresponding versions of TM namely, TM-1 to TM-5. We then compute the error of all models i.e. CM-1,..., CM-5 and TM, TM-1, ..., TM-5 on the MNIST test set of 10K examples. We use train2 to train the TM so as to get better estimates of confidence scores from the CMs as opposed to trusting an overfitted model. Moreover, this also gives us better resolution of weights from the CMs, as most of them have confidence scores close to 1 for almost all the train1 instances. We do this for 10 random splits and the mean error with 95\% confidence intervals are reported.

\subsubsection{Observations}

We see in figure \ref{fig:mnist} (left) that the most complex CM which is CM-1, has the best test performance. The performance drops monotonically as the CMs become less complex. From the TMs perspective we see that all the CMs help in reducing its test error. However, TM-4 has the lowest error amongst the TMs, which corresponds to CM-4. Thus, even though CM-4 is not the best performing CM it is the best teacher for the TM given our interpretable strategy. Consequently, in our framework, CM-4 is (0.95, 0)-interpretable, while CM-1 is (0.98, 0)-interpretable, with others lying in between.

To see why this happens we plot the distribution of the weights that are obtained by each CM which is observed in figure \ref{fig:mnist} (right). We see that the complicated CM is so good that for almost 98\% of the instances it has a confidence score of $\sim$ 1. The distribution starts to become more spread out as the complexity of the CMs reduces.  

\subsubsection{Insight}

So what insight do the above observations convey. \emph{Given our interpretable strategies of weighting instances the improvement in the TM is a function of the performance of the CM and the diversity in its confidence scores. If the CM is so good that all its confidence scores are close to 1 then almost no new information is passed to the TM as the weighted training set is practically equivalent to the original unweighted one.}

This leads to the following qualitative insight.

\par\vspace{0.25cm}
\fbox{\begin{minipage}{4.5 in}\label{qlinsight} 
Having a teacher who is exceptional in an area may not be the best for the student as the teacher is not able to resolve what may be more difficult as opposed to less difficult and is thus unable to provide extra information that may give direction to help the student.
\end{minipage}}
\par\vspace{0.25cm}

Of course, all of the above is contingent on the interpretable strategy and there may be better ways to extract information from complex models such as CM-1. Nonetheless, we feel the above point is thought provoking.

\subsection{Evaluation on CIFAR-10}

We now conduct experiments on CIFAR-10, where the complex model is a 18 unit ResNet \cite{resnet} and we create four smaller models TM-1, ..., TM-4 which have 3, 5, 7 and 9 ResNet units respectively.

\subsubsection{Setup}

The complex model has $15$ Resnet units in sequence. The basic blocks each consist of two consecutive $3\times 3$ convolutional layers with either 64, 128, or 256 filters and our model has five of each of these units.\eat{(termed $\mathrm{Resunit:}$1-$\mathrm{x}$, $\mathrm{Resunit:}$2-$\mathrm{x}$, and $\mathrm{Resunit:}$3-$\mathrm{x}$, correspondingly, where $\mathrm{x}\in\{0,1,2,3,4\}$ because the complex model uses 5 of each type of block).} The first Resnet unit is preceded by an initial $3\times 3$ convolutional layer with $16$ filters. The last Resnet unit is succeeded by an average pooling layer followed by a fully connected layer producing $10$ logits, one for each class. Details of the $15$ Resnet units are given in Appendix B.

All target models have the same initial convolutional layer and finish with the same average pooling and fully connected layers as in the complex model above. We have four target models with 3, 5, 7, and 9 ResNet units. The approximate relative sizes of these models to the complex neural network are $1/5$, $1/3$, $1/2$, $2/3$, correspondingly. Further details can be found in Appendix B.

We split the available $50000$ training samples from the CIFAR-10 dataset into training set $1$ consisting of $30000$ examples and training set $2$ consisting of $20000$ examples. We split the $10000$ test set into a validation set of $500$ examples and a holdout test set of $9500$ examples. All final test accuracies of the simple models are reported with respect to this holdout test set. The validation set is used to tune all models and hyperparameters.

The complex model is trained on training set $1$.  We obtained a test accuracy of $0.845$ and keep this as our complex model. We note that although this is suboptimal with respect to Resnet performances of today, we have only used $30000$ samples to train. Each of the target models are trained only on training set $2$ consisting of $20000$ samples for $500$ epochs. All training hyperparameters are set to be the same as in the previous cases.  We train each target model in Table 4 (Appendix B) for the following different cases. We train a standard unweighted model. We then train using our MLE procedure. Distillation \cite{distill} trains the target models using cross-entropy loss with soft targets obtained from the softmax outputs of the complex model's last layer rescaled by temperature $t=0.5$ (tuned with cross-validation). More details along with results for different temperatures for distillation are given in the Appendix C.

The results are averaged over 4 runs and 95\% confidence intervals are provided.

\begin{table*}
\centering
 \begin{tabular}{|c|c|c|c|c|}
  \hline
    \hfill & TM-1 & TM-2 & TM-3 & TM-4 \\ 
    \hline
     Unweighted & 73.15($\pm$ 0.7) & 75.78($\pm$0.5) & 78.76($\pm$0.35) & 79.90($\pm$0.34) \\
     \hline
Weighted (MLE) & \textbf{76.27} ($\pm$0.48) & \textbf{78.54} ($\pm$0.36) & \textbf{81.46}($\pm$0.50) & \textbf{82.09} ($\pm$0.08) \\
\hline
Distillation & 65.84($\pm$0.60) & 70.09 ($\pm$0.19)
&73.4($\pm$0.64)& 77.30 ($\pm$0.16) \\
\hline
 \end{tabular}
 \caption{Averaged accuracies (\%) of target model trained with various methods. The complex model achieved $84.5 \%$ accuracy. In each case, the improvement over the unweighted model is about $2-3\%$ in test accuracy. Distillation performs uniformly worse in all cases.}
 \label{tab:acc}
\vspace{-0.5cm}
\end{table*}

\subsubsection{Observations}

In Table \ref{tab:acc}, we observe the performance of the various methods. We see that our weighting scheme improves the performance of the 4 TMs by roughly 2-3\% over standard unweighted training. This implies that our method corresponding to the 18 unit ResNet was $(0.88, 0)$-interpretable, $(0.88, 0)$-interpretable, $(0.87, 0)$-interpretable and $(0.89, 0)$-interpretable for TM-1 to TM-4 respectively.

Distillation performs uniformly worse, where the performance compared even with the original unweighted training is subpar leading to a $\delta > 1$. It seems from this that Distillation is not that effective when it comes to improving shallower or simpler models, rather is most effective when the depth of the original network is maintained with some thinning down of intermediate layers \cite{fitnet}.

\begin{figure}[t]
  \centering  
      \includegraphics[width=0.8\textwidth]{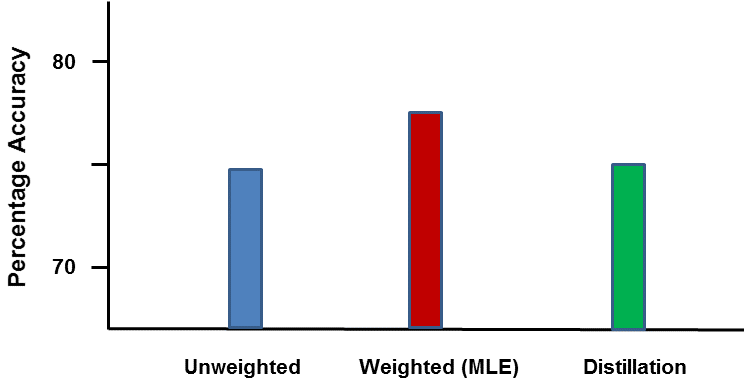}  
  \caption{Above we show the performance of the different methods on the manufacturing dataset.}
  \label{manuf}
\vspace{-0.5cm}
\end{figure}

\subsection{Evaluation on Manufacturing Data}

We now describe a real scenario from the semi-conductor manufacturing domain where a neural network was used to improve a decision tree CART model. The decision tree was the domain experts preference as it is something that was understandable to him.

\subsubsection{Setup}

We consider an etching process in a semi-conductor manufacturing plant. The goal is to predict the quantity of metal etched on each wafer -- which is a collection of chips -- without having to explicitly measure it using high precision tools, which are not only expensive but also substantially slow down the throughput of a fab. If $T$ denotes the required specification and $\gamma$ the allowed variation, the target we want to predict is quantized into three bins namely: $(-\infty,T-\gamma)$, $(T+\gamma,\infty)$ and within spec which is $T\pm\gamma$. We thus have a three class problem and the engineers goal is not only to predict these classes accurately but also to obtain insight into ways that he can improve his process.

For each wafer we have 5104 input measurements for this process. The inputs consist of acid concentrations, electrical readings, metal deposition amounts, time of etching, time since last cleaning, glass fogging and various gas flows and pressures. The number of wafers in our dataset was 100,023. Since these wafers were time ordered we split the dataset sequentially where the first 70\% was used for training and the most recent 30\% was used for testing. Sequential splitting is a very standard procedure used for testing models in this domain, as predicting on the most recent set of wafers is more indicative of the model performance in practice than through testing using random splits of train and test with procedures such as 10-fold cross validation.

\subsubsection{Observations}

We built a neural network (NN) with an input layer and five fully connected hidden layers of size 1024 each and a final softmax layer outputting the probabilities for the three classes. The NN had an accuracy of 91.2\%. The NN was, however, not the model of choice for the fab engineer who was 
more familiar and comfortable using decision trees.

Given this, we trained a CART based decision tree on the dataset. As seen in figure \ref{manuf}, its accuracy was 74.3\%. We wanted to see if we could improve this accuracy and build a more accurate decision tree. We thus applied our MLE method to weight the instances and then retrain CART. This bumped its accuracy to 77.1\% implying that our method was $(0.89, 0)$-interpretable. The best distilled CART produced a slight improvement in the base model increasing its accuracy to 75.6\%, which implies that it was $(0.95, 0)$-interpretable.

\begin{table*}[t]
\centering
 \begin{tabular}{|c|c|c|c|c|}
  \hline
    \hfill & CM (Xgboost) & TM-1 (Decision Tree)  & TM-2 (Naive Bayes) \\ 
    \hline
     Unweighted-Accuracy & 73.4 & 71.1& 66.9 \\
     \hline
    Weighted (MLE)-Accuracy & N/A & \textbf{71.6} & \textbf{67.8} \\
     \hline 
Unweighted- AUC & 80.1& 76.8 & 75.9 \\
\hline
Weighted (MLE)-AUC & N/A & \textbf{77.6} & \textbf{76.1} \\
\hline
 \end{tabular}
 \caption{Accuracy and ROC-AUC (\%) of two target models trained with various methods on the FICO-Challenge Dataset. The complex model achieved $73.4 \%$ accuracy. For each of the target models, the improvement over the unweighted model is about $0.5-1\%$ in test accuracy and ROC-AUC. The difference in accuracies between target models and complex models is about $2-3\%$. Even then, the improvement of weighted models over unweighted is roughly about $25 \%$ of the difference between complex and target model in the test metrics. }
 \label{tab:acc1}
\end{table*}

\subsection{Evaluation on FICO-Challenge Dataset}
 We now describe results on training a complex model (Xgboost classifier) and other interpretable target models (Naive Bayes Classifier and Decision Tree Classifier) on the FICO Challenge Dataset (also called the HELOC Dataset) \cite{FICO}.  
 
\subsubsection{SETUP}
   The dataset consists of anonymized applications for Home Equity Line of Credit (HELOC) from real home owners. HELOC is a specific type of line of credit offered by a bank which is a certain percentage of the difference between the current market value and purchase price of a home. The customers in the dataset have requested credit in the range $[\mathrm{USD~} 5000, \mathrm{USD~} 150000]$. The target variable to predict is `Risk Performance', which indicates whether the customer's payment was past 90 days due and/or was at a worse condition in the past 2 years. The Risk Performance is a binary variable - `Good' or `Bad'. 
   
\subsubsection{Feature Processing}   
The set of features has both categorical and real valued features. Categorical features were one hot encoded. Additionally, irrespective of the type of feature (real valued or categorical), three special values may appear in the dataset. $-9$ means that no credit history is found, and we removed these rows since they just contain a row of all $-9$s. $-7$ means ` Condition Not Met'; in other words, no delinquencies and/or no inquiries. $-8$ means that there are no usable inquiries or account trades. These special values may be informative about the target variable. Therefore, for every feature column we create three feature columns. If the value is $-7$, the three features created are set to $1,0$ and $0$. When the value is $-8$, it is set to $0,1,0$ and when the feature takes a non-special value (neither $-7$ nor $-8$), then the first two features are set to $0$ and the last one is set to the feature value. 

\subsubsection{Training}
We train three classifiers: a) Xgboost classifier which is our complex model (CM), b) Naive Bayes Classifier (TM-1), and c) Decision Tree Classifier (TM-2). The latter two classifiers are interpretable. Each classifier is trained using standard APIs from scikit-learn (version 0.20). We used $70\%$ of the samples for training and validation (5-fold cross-validation), and the remaining $30\%$ as a holdout test set.

After cross-validation, the best set of parameters found for Xgboost (CM) were the following: Learning rate: $0.05$, max depth: $3$, minimum child weight: $1$, number of boosting rounds: $300$, $\ell_1$ penalty: $5$, $\ell_2$ penalty: $10$. 

Two versions of the Decision Tree classifier were trained. One was the unweighted version and the other was the weighted version where each sample was weighted by the confidence score of the label in the training set according to the Xgboost classifier (CM). Similarly, two versions of Naive Bayes Classifier were trained. We note that the training procedure, in terms of hyper-parameter optimization, was identical for both weighted and unweighted versions of both target models.

\subsubsection{Observations}
The test accuracies and ROC-AUC values are given in Table \ref{tab:acc1}. The complex model (Xgboost) achieved $73.4 \%$ accuracy and an ROC-AUC of $80.1\%$ on the test dataset. For each target model, the improvement over the unweighted model is about $0.5-1\%$ in test accuracy and ROC-AUC. The difference in accuracies between target models and complex models is about $2-3\%$. This shows that, even for small differences in accuracies between target and complex models, the target accuracy can be improved. Xgboost is thus $(0.98,0)$-interpretable and $(0.97,0)$-interpretable with respect to Decision Tree  and Naive-Bayes, respectively, when considering test accuracies, while it is $(0.96,0)$-interpretable and $(0.99,0)$-interpretable when ROC-AUC is the metric.

\section{Related Work}

There has been a great deal of interest in interpretable modeling recently and for good reason. In almost any practical application with a human decision maker, interpretability is imperative for the human to have confidence in the model. It has also become increasingly important in deep neural networks given their susceptibility to small perturbations that are humanly unrecognizable \cite{carlini,gan}. 

There have been multiple frameworks and algorithms proposed to perform interpretable modeling. These range from building rule/decision lists \cite{decl,twl} to finding prototypes \cite{l2c} to taking inspiration from psychometrics \cite{irt} and learning models that can be consumed by humans \cite{Caruana:2015}. There are also works \cite{lime} which focus on answering instance specific user queries by locally approximating a superior performing complex model with a simpler easy-to-understand one which could be used to gain confidence in the complex model. Authors in \cite{bastani2017interpreting} propose an interpretable procedure, to transfer information from any classifier to a decision tree improving a baseline decision tree. However, the procedure is specific to axis-aligned decision trees as the target model. There is also recent work \cite{unifiedPI} which proposes a unified approach to create local model explanations with certain desirable properties that many current methods seem to lack. 

A recent survey \cite{montavon2017methods} looks at primarily two methods for neural network understanding: a) Methods \cite{nguyen2016synthesizing,nguyen2016multifaceted} that produce a prototype for a given class by optimizing the confidence score for the class subject to some regularization on the prototype (viz. constraints based on range space of a GAN), b) Explaining a neural net's decision on an image by highlighting relevant parts using a technique called Layer-wise relevance propagation \cite{bach2015pixel}. This technique starts from the last layer and progressively assigns weights to neurons of layers below connected to a single neuron on a layer above satisfying some weight conservation properties across layers. We observe that type (b) methods are local model explanations on a specific image while type (a) methods are more global producing prototypes for a given class. Other works also investigate methods of the type (b) discussed above for vision \cite{selvaraju2016grad} and NLP applications \cite{lei2016rationalizing}. These methods also fall within our framework.

A very relevant work to our current endeavor is possibly \cite{rsi}. They provide an in depth discussion for why interpretability is needed, and an overall taxonomy for interpretability. The primary focus is on direct human interpretability. We on the other hand, use a TM, which could be a human, to define our notion of interpretability making it more quantifiable/measurable.

A recent position paper \cite{lipton2016mythos} lists several desirable properties of interpretable methods. Notably, the author also emphasizes the need for a proper formalism for notion of interpretability and quantifying methods based on these formalisms. We have done precisely that through our formalism for $\delta$-interpretability.

Our interpretable procedures based on using confidence measures are related to distillation and learning with privileged information \cite{priv16}. The key difference is in the way we use information. We weight training instances by the confidence score of the training label alone. This approach, unlike Distillation \cite{distill}, is applicable in broader settings like when target models are classifiers optimized using empirical risk (e.g., SVM) where risk could be any loss function.\eat{
In contrast, Distillation is used when the target model is a neural network that is trained with a cross-entropy penalty on the vector of confidence scores of the complex model. However, this is not feasible when target models are classifiers that are optimized using the empirical risk (like in SVM etc.) where risk could be any loss function.
} Moreover, weighting instances has an intuitive justification where if you view the complex model as a teacher and the TM as a student, the teacher is telling the student which easier aspects (e.g. instances) he/she should focus on and which he/she could ignore. There have been other strategies \cite{modelcompr,modelcompr2,bastani2017interpreting} to transfer information from bigger models to smaller ones, however, they are all similar in spirit to Distillation, where the complex models predictions are used to train a simpler model.

Weighting examples to improve models has been used before, although their general setup and motivation is different, for instance curriculum learning (CL) \cite{curriculumL} and boosting \cite{boost}. CL is a training strategy where first easy examples are given to a learner followed by more complex ones. The determination of what is simple as opposed to complex is typically done by a human. There is usually no automatic gradation of examples that occurs based on a machine. Also sometimes the complexity of the learner is increased considerably during the training process so that it can accurately model more complex phenomena. In our case however, the complexity of the target model is (more or less) fixed given applications in interpretability. Moreover, we are searching for just one set of weights which when applied to the original input (not some intermediate learned representations) the fixed simple model trained on it gives the best possible performance. Boosting is even more remotely related to our setup. In boosting there is no high performing teacher and one generally grows an ensemble of weak learners which as just mentioned is not reasonable in our setting. In addition here too, there are multiple sets of weights where each learner is trained on a different set.

Other relevant literatures where information is transfered between two or more entities can be witnessed in multi-agent coordination \cite{multiagent} and machine teaching \cite{machteach}. In multi-agent coordination, the goal typically is to drive the independent agents towards a consensus through efficient message passing given the lack of global information available to each of them. The main objective is not necessarily to improve each agents performance, rather to make them coordinate to as to solve a common problem in a distributed environment. In machine teaching, the goal is to teach a student model a certain concept that the teacher is quite familiar with. One of the main formalisms here is to convey this concept using the least number of examples. An implicit assumption in many of the cases here is that the student and teacher are working in the same parameter space and the student has the ability to learn the optimal concept. In our setting, there is no specific concept that we wish to convey, rather we have proposed a metric to quantify interpretability and our candidate methods provide a way to transfer information to a target model in a soft sense (eg. weighting) with no hard requirement of minimizing the number of examples. Moreover, the student and the teacher can have significantly different hypothesis classes. As long as the information is consumable by the target model with impact measured using our definitions the corresponding procedure would be interpretable with sparsity not explicitly enforced, although, it may be important for consumption for certain target models (viz. humans).

\section{Discussion}

In this paper we provided a formal framework to characterize interpretability. Using this framework we were able to quantify the performance of many state-of-the-art interpretable procedures. We also proposed our own for the supervised setting that are based on strong theoretical grounding.

Our experiments led to the insight that having the best performing complex model is not necessarily the best in terms of improving a TM. In other words, it seems important to characterize the relative complexity of a (CM, TM) pair for useful information transfer. Trying to characterize this is part of future work. Of course, all of this is relative to the interpretable strategies that one can come up with. Hence, in the future we also want to design better interpretable strategies for more diverse settings.

From an information theoretic point of view, our work motivates the following two kinds of capacity notions:
a) What is the least number of additional bits per training sample required for the TM to improve its performance by $\delta$? These additional bits would reduce the uncertainty in the confidence score about a target label than what is implied by the training data.
b) What is the maximum number of additional bits per training sample that can be obtained from the CM towards reducing the uncertainty of the confidence scores of the TM? Based on these two questions, it may be possible to say that when the second capacity exceeds the first capacity, then a $\delta$ improvement is possible. We intend to investigate this in the future.

We defined $\delta$ for a single distribution but it could be defined over multiple distributions where $\delta = \max(\delta_1,...,\delta_k)$ for $k$ distributions and analogously $\gamma$ also could be defined over multiple adversarial distributions. We did not include these complexities in the definitions so as not to lose the main point, but extensions such as these are straightforward.

Another extension could be to have a sensitivity parameter $\alpha$ to define equivalence classes, where if two models are $\delta_1$- and $\delta_2$-interpretable, then they are in the same equivalence class if $|\delta_1-\delta_2|\le \alpha$. This can help group together models that can be considered to be equivalent for the application at hand. The $\alpha$ in essence quantifies operational significance. One can have even multiple $\alpha$ as a function of the $\delta$ values.

Regarding complexity, one can also extend the notion of interpretability where $\delta$ and/or $\gamma$ are 1 but you can learn the same model with fewer samples given information from the interpretable procedure. In essence, have sample complexity also as a parameter in the definition. Furthermore, Feldman~\cite{feldman2000minimization} finds that the subjective difficulty of a concept is directly proportional to its Boolean complexity, the length of the shortest logically equivalent propositional formula. We could explore interpretable models of this type. Yet another model bounds the rademacher complexity of humans \cite{zhu2009human} as a function of complexity of the domain and sample size. Although the bounds are loose, they follow the empirical trend seen in their experiments on words and images.

\eat{
One can also extend the notion of interpretability where $\delta$ and/or $\gamma$ are 1 but you can learn the same model with fewer samples given information from the interpretable procedure. In essence, have sample complexity also as a parameter in the definition.

Human subjects store approximately 7 pieces of information \cite{lisman1995storage}. As such, we can explore highly interpretable models, which can be readily learned by humans, by considering models for TM that make simple use of no more than 7 pieces of information.

Feldman~\cite{feldman2000minimization} finds that the subjective difficulty of a concept is directly proportional to its Boolean complexity, the length of the shortest logically equivalent propositional formula. We could explore interpretable models of this type. 

Another model bounds the rademacher complexity of humans \cite{zhu2009human} as a function of complexity of the domain and sample size. Although the bounds are loose, they follow the empirical trend seen in their experiments on words and images.
}

We lastly discuss human subjects, which are known store approximately 7 pieces of information \cite{sevenmiller, lisman1995storage}. As such, we can explore highly interpretable models, which can be readily learned by humans, by considering models for TM that make simple use of no more than 7 pieces of information. Finally, all humans may not be equal relative to a task. Having expertise in a domain may increase the level of detail consumable by that human. So the above models which try to approximate human capability may be extended to account for the additional complexity consumable by the human depending on their experience.

\section*{Acknowledgements}
We would like to thank Margareta Ackerman, Murray Campbell, Alexandra Olteanu, Marek Petrik, Irina Rish, Kush Varshney, Mark Wegman and Bowen Zhou for their suggestions. We would also like to thank the anonymous reviewers for their insightful comments.

\appendix
\section{Analysis: Proof of Theorem \ref{mainthm}}
We assume that any model for a given binary classification task is essentially a conditional probability distribution function $p(y|x),~ y \in \{+1.-1\}$. All classifiers are assumed to follow the following classification rule: $\argmax \{p(y|x):y \in \{+1,-1\}\}$. Let us denote the conditional probability distribution function for the complex model as $p_{\mathrm{CM}}(y|x)$. Let us denote the conditional probability distribution function for the target model with parameter $\mathbf{\theta}$ as $p_{\mathrm{TM}} (y|x; \mathbf{\theta})$.We unify the treatment through the lens of conditional probability scores. So one must define an explicit conditional probability score for the risk minimization models. We do so in the following subsection.

\subsection{Risk Minimization models and pseudo-confidence scores}
Suppose the target model is optimized according to empirical risk minimization on $m$ training samples using the risk function $r(y,x,\mathbf{\theta})$, i.e.
\begin{align}
   \min \limits_{\theta} \frac{1}{m} \sum_{i=1}^m r(y_i,x_i,\mathbf{\theta}) 
\end{align}
Let us assume that $0 \leq r(\cdot) \leq 1$. Let the shorthand notation $r_1(x)$ denote $r(+1,x,\mathbf{\theta})$ while $r_2(x)$ denote $r(-1,x,\mathbf{\theta})$. Define the conditional probability distribution on the target model based on the risk function as follows:
\begin{align}
 p_{\mathrm{TM}} (+1|x;\mathbf{\theta})= \frac{e^{-cr_1(x)}}{e^{-cr_1(x)}+e^{-cr_2(x)}} 
\end{align}
for some constant $c>0$. Please note that, no matter what $c$ is, the behavior of the classifier on actual data depends on whether $r_1(x)>r_2(x)$ or not. Therefore for all $c>0$, this is equivalent to checking if $ p_{\mathrm{TM}} (+1|x;\mathbf{\theta}) > 1/2$ or not. In fact, the behavior of the error term at the LHS of  (\ref{errordecomp}) depends only on whether $r_1(x)>r_2(x)$ or not and this is independent of the choice of $c$. So here, we don't attach any real notion of confidence score to $p_{\mathrm{TM}}(\cdot)$ defined as above. They can be considered to be pseudo-confidence scores implied by the risk function $r(\cdot)$ for the sake of analysis. So any risk function on the target model endows it a pseudo-confidence score.
\eat{
\subsection{Maximum Likelihood Estimation Models}
In these cases, the binary classifier is specified directly by a conditional probability distribution $p_{\mathrm{TM}}(y|x; \mathbf{\theta})$. Given $m$ training samples $(y_i,x_i)$, the following optimization is performed:
 \begin{align}
     \min \limits_{\theta} \frac{1}{m} \sum_{i=1}^m - \log p_{\mathrm{TM}}(y_i|x_i; \mathbf{\theta})
 \end{align}
}
\subsection{Error Term}
We will always treat an ERM case as an MLE case endowed with  pseudo-confidence scores. We first note the following: The best GE one can obtain if the data is distributed according to ${\cal D}_{\mathrm{CM}}$ is exactly $\mathbb{E}_{{\cal D}_{\mathrm{CM}}}[ \mathbf{1}_{p_{\mathrm{CM}}(y|x)<=1/2} ]$. This is because, even if the classifier knows the correct distribution, given that the output is either $+1,-1$, there will be some error due to this quantization. We wish to find the optimum $\mathbf{\theta}$ that minimizes the classification error of the target model assuming that the samples arise from ${\cal D}_{\mathrm{CM}}$. Based on the above observation, we split this error into two terms:
\begin{align}\label{errordecomp}
 \mathbb{E}_{{\cal D}_{\mathrm{CM}}}[ \mathbf{1}_{p_{\mathrm{TM}}(y|x;\mathbf{\theta})<=1/2} ] &= \mathbb{E}_{{\cal D}_{\mathrm{CM}}}[ \mathbf{1}_{p_{\mathrm{TM}}(y|x;\mathbf{\theta})<=1/2} ] - \mathbb{E}_{{\cal D}_{\mathrm{TM},\mathbf{\theta}}}[ \mathbf{1}_{p_{\mathrm{TM}}(y|x;\mathbf{\theta})<=1/2} ]  \nonumber \\
 \hfill & + \mathbb{E}_{{\cal D}_{\mathrm{TM},\mathbf{\theta}}}[ \mathbf{1}_{p_{\mathrm{TM}}(y|x;\mathbf{\theta})<=1/2} ] 
\end{align}
The second term is the residual error of the perfect classifier on samples drawn according to the distribution defined by the target model, i.e. ${\cal D}_{\mathrm{TM},\mathbf{\theta}} = p(x) p_{\mathrm{TM}} (y|x;\mathbf{\theta})$. 

\subsection{Bounding the first difference term}
Now, we bound the first difference term in (\ref{errordecomp}) as follows:
\begin{theorem}
\begin{align} 
 \mathbb{E}_{{\cal D}_{\mathrm{CM}}}[ \mathbf{1}_{p_{\mathrm{TM}}(y|x;\mathbf{\theta})<=1/2} ] - \mathbb{E}_{{\cal D}_{\mathrm{TM},\mathbf{\theta}}}[ \mathbf{1}_{p_{\mathrm{TM}}(y|x;\mathbf{\theta})<=1/2} ] \leq \sqrt{\frac{1}{2} \mathrm{KL}( p_{\mathrm{CM}}(y|x) \lVert p_{\mathrm{TM}}(y|x;\mathbf{\theta} ) )} 
 \end{align}
\end{theorem}
\begin{proof}
Let $d_{\mathrm{TV}}(p,q)$ be the total variation distance between two distributions $p$ and $q$. We have the following simple chain of inequalities: 
 \begin{align}
  \mathbb{E}_{{\cal D}_{\mathrm{CM}}}[ \mathbf{1}_{p_{\mathrm{TM}}(y|x;\mathbf{\theta})<=1/2} ] - \mathbb{E}_{{\cal D}_{\mathrm{TM},\mathbf{\theta}}}[ \mathbf{1}_{p_{\mathrm{TM}}(y|x;\mathbf{\theta})<=1/2} ] & \overset{a}\leq d_{\mathrm{TV}} ({\cal D}_{\mathrm{TM},\mathbf{\theta}},{\cal D}_{\mathrm{CM}}) \nonumber \\
  \hfill & \overset{b} \leq \sqrt{\frac{1}{2} \mathrm{KL}( {\cal D}_{\mathrm{CM}} \lVert {\cal D}_{\mathrm{TM},\mathbf{\theta}} )} \nonumber \\
    \hfill &= \sqrt{\frac{1}{2} \mathrm{KL}( p_{\mathrm{CM}}(y|x) \lVert p_{\mathrm{TM}}(y|x;\mathbf{\theta} ) )} 
 \end{align}
 (a)- follows from the definition of total variation distance. (b) follows from Pinsker's inequality connecting KL-divergence and total variation distance. This completes the proof.
\end{proof}

\begin{theorem}

 \textbf{MLE case:}
 \begin{align}
 \mathrm{KL}( p_{\mathrm{CM}}(y|x) \lVert p_{\mathrm{TM}}(y|x;\mathbf{\theta} ) )  & \leq  \mathbb{E}_{p(x)} [ -(p_{\mathrm{CM}}(+1|x)) \log( p_{\mathrm{TM}}(+1|x;\mathbf{\theta} )) ] + \nonumber \\
 \hfill & \mathbb{E}_{p(x)}[-p_{\mathrm{CM}}(-1|x)\log( p_{\mathrm{TM}}(-1|x;\mathbf{\theta} ))]
 \end{align}
 \textbf{ERM case (a):}
 \begin{align}
  \mathrm{KL}( p_{\mathrm{CM}}(y|x) \lVert p_{\mathrm{TM}}(y|x;\mathbf{\theta} ) )  & \leq  \mathbb{E}_{p(x)} [(p_{\mathrm{CM}}(+1|x)) c r_1(x) + (p_{\mathrm{CM}}(-1|x)) c r_2(x)]  \nonumber \\
 \hfill & + \mathbb{E}_{p(x)} [\log (1 +e^{-c \lvert r_1(x) - r_2(x) \rvert}) ]
 \end{align}
\end{theorem}
\begin{proof}
 We have the following chain of inequalities:
 
 \begin{align}
   \mathrm{KL}( p_{\mathrm{CM}}(y|x) \lVert p_{\mathrm{TM}}(y|x;\mathbf{\theta} ) ) &= \mathbb{E}_{{\cal D}_{\mathrm{CM}}}[\log p_{\mathrm{CM}}(y|x)]  +  \mathbb{E}_{p(x)} [-p_{\mathrm{CM}}(+1|x) \log( p_{\mathrm{TM}}(+1|x;\mathbf{\theta} )) \nonumber \\ 
 \hfill &  - p_{\mathrm{CM}}(-1|x) \log( p_{\mathrm{TM}}(+1|x;\mathbf{\theta} )) ] \nonumber \\  
      & \overset{a} \leq   \mathbb{E}_{p(x)} [-p_{\mathrm{CM}}(+1|x) \log( p_{\mathrm{TM}}(+1|x;\mathbf{\theta} )) \nonumber \\ 
 \hfill &  - p_{\mathrm{CM}}(-1|x) \log( p_{\mathrm{TM}}(+1|x;\mathbf{\theta} )) ]
 \end{align}
 (a)- This is because $\log (p_{\mathrm{CM}}(\cdot) ) \leq 0$. This proves the result for the MLE case.For the ERM case, we further bound using risk functions.
 
 \begin{align}
   \mathrm{KL}( p_{\mathrm{CM}}(y|x) \lVert p_{\mathrm{TM}}(y|x;\mathbf{\theta} ) ) &= \leq   \mathbb{E}_{p(x)} [-p_{\mathrm{CM}}(+1|x) \log( p_{\mathrm{TM}}(+1|x;\mathbf{\theta} )) \nonumber \\ 
 \hfill &  - p_{\mathrm{CM}}(-1|x) \log( p_{\mathrm{TM}}(+1|x;\mathbf{\theta} )) ] \nonumber \\
 \hfill & =   \mathbb{E}_{p(x)} [p_{\mathrm{CM}}(+1|x) c r_1(x)] + \mathbb{E}_{p(x)}[ p_{\mathrm{CM}}(-1|x) c r_2(x)]  \nonumber \\
  \hfill & + \mathbb{E}_{p(x)} [\log(e^{-cr_1(x)}+e^{-cr_2(x)})] \nonumber \\
  \hfill &\leq \mathbb{E}_{p(x)} [(p_{\mathrm{CM}}(+1|x)) c r_1(x)] +  \mathbb{E}_{p(x)} [(p_{\mathrm{CM}}(-1|x)) c r_2(x) \nonumber \\
 \hfill & -c \min (r_1(x),r_2(x))+ \log (1 +e^{-c \lvert r_1(x) - r_2(x) \rvert}) ] \nonumber \\
  \hfill &\leq \mathbb{E}_{p(x)} [(p_{\mathrm{CM}}(+1|x)) c r_1(x)] +  \mathbb{E}_{p(x)} [(p_{\mathrm{CM}}(-1|x)) c r_2(x) \nonumber \\
 \hfill & + \log (1 +e^{-c \lvert r_1(x) - r_2(x) \rvert}) ]
  \end{align}
 The last inequality proves the ERM part of the theorem. 
\end{proof}

\subsection{Bounding the second term}
\textbf{ERM case (a):}
The second term in (\ref{errordecomp}) can be expressed as follows:
\begin{align}
 \mathbb{E}_{{\cal D}_{\mathrm{TM},\mathbf{\theta}}}[ \mathbf{1}_{p_{\mathrm{TM}}(y|x;\mathbf{\theta})<=1/2} ] &= 
 \mathbb{E}_{p(x)}\left[ \min(\frac{e^{cr_1(x)}}{e^{cr_1(x)}+e^{cr_2(x)}},\frac{e^{cr_2(x)}}{e^{cr_1(x)}+e^{cr_2(x)}}) \right] \nonumber \\
 \hfill &= \mathbb{E}_{p(x)} \left[ \frac{1}{1+e^{c\lvert r_1(x)-r_2(x)\rvert}} \right] \leq \mathbb{E}_{p(x)} [e^{-c\lvert r_1(x) - r_2(x) \rvert}]
\end{align}

\textbf{MLE case:} We bound the second term in (\ref{errordecomp}) as follows:
\begin{align}
 \mathbb{E}_{{\cal D}_{\mathrm{TM},\mathbf{\theta}}}[ \mathbf{1}_{p_{\mathrm{TM}}(y|x;\mathbf{\theta})<=1/2} ] &= 
 \mathbb{E}_{p(x)}\left[ \min(p_{\mathrm{TM}}(+1|x;\mathbf{\theta}),p_{\mathrm{TM}}(-1|x;\mathbf{\theta})) \right] \nonumber \\
 \hfill & \leq \mathbb{E}_{p(x)} \left[ \frac{\min(p_{\mathrm{TM}}(+1|x;\mathbf{\theta}),p_{\mathrm{TM}}(-1|x;\mathbf{\theta}))}{\max(p_{\mathrm{TM}}(+1|x;\mathbf{\theta}),p_{\mathrm{TM}}(-1|x;\mathbf{\theta}))} \right] \nonumber \\
 & \leq \mathbb{E}_{p(x)} [e^{-\lvert \log p_{\mathrm{TM}}(-1|x;\mathbf{\theta}) - \log p_{\mathrm{TM}}(+1|x;\mathbf{\theta}) \rvert}]
\end{align}

\subsection{Bounding the error term: Putting it together}

Therefore, we put everything together minimize the following upper bound on the square of the TM error with respect to the CM model as a function of $\mathbf{\theta}$.

\begin{proof}[Proof of Theorem \ref{mainthm}]

\textbf{ERM case (a):}
From (\ref{errordecomp}), we have:
\begin{align}
 (\mathbb{E}_{{\cal D}_{\mathrm{CM}}}[ \mathbf{1}_{p_{\mathrm{TM}}(y|x;\mathbf{\theta})<=1/2} ])^2 &\leq 2 (\mathbb{E}_{{\cal D}_{\mathrm{CM}}}[ \mathbf{1}_{p_{\mathrm{TM}}(y|x;\mathbf{\theta})<=1/2} ] - \mathbb{E}_{{\cal D}_{\mathrm{TM},\mathbf{\theta}}}[ \mathbf{1}_{p_{\mathrm{TM}}(y|x;\mathbf{\theta})<=1/2} ])^2  \nonumber \\
 \hfill & + 2(\mathbb{E}_{{\cal D}_{\mathrm{TM},\mathbf{\theta}}}[ \mathbf{1}_{p_{\mathrm{TM}}(y|x;\mathbf{\theta})<=1/2} ])^2 \nonumber \\
    & \leq \left[ \mathbb{E}_{p(x)} [(p_{\mathrm{CM}}(+1|x)) c r_1(x) + (p_{\mathrm{CM}}(-1|x)) c r_2(x)] \right. \nonumber \\
 \hfill & \left. + \mathbb{E}_{p(x)} [\log (1 +e^{-c \lvert r_1(x) - r_2(x) \rvert}) ] \right] + 2 (\mathbb{E}_{p(x)} [e^{-c\lvert r_1(x) - r_2(x) \rvert}])^2 \nonumber \\
  \hfill & \overset{a} \leq  \mathbb{E}_{p(x)} \left[(p_{\mathrm{CM}}(+1|x)) c r_1(x) + (p_{\mathrm{CM}}(-1|x)) c r_2(x) \right. \nonumber \\
 \hfill & \left. + \log (1 +e^{-c \lvert r_1(x) - r_2(x) \rvert}) + 2e^{-2c\lvert r_1(x) - r_2(x) \rvert}\right]
\end{align}
(a) - Jensen's inequality on the convex function $x^2$. Similarly, one can show the result for the MLE case.
\end{proof}


For ERM case (b), we provide the following analysis of the error of the target model assuming that the data is drawn according to the distribution ${\cal D}_{\mathrm{CM}}=p(x) p_{\mathrm{CM}}(y|x)$. Let $y'(x)= \argmax \limits_{y} p_{\mathrm{CM}}(y|x)$. Consider the the error of the target model in the ERM case:

\begin{align}\label{errorbound}
\mathbb{E}_{{\cal D}_{\mathrm{CM}}}[\mathbf{1}_{r(y,x,\theta)>r(-y,x,\theta)}] &= \mathbb{E}_{x} \left[ \left[\frac{1}{2}+ \left| \frac{1}{2}-p_{\mathrm{CM}}(y'(x)|x) \right| \right]\cdot\mathbf{1}_{r(y'(x),x,\theta)> r(-y'(x),x,\theta)}+ \right. \nonumber \\ \hfill & \left. \left[\frac{1}{2}- \left| \frac{1}{2}-p_{\mathrm{CM}}(y'(x)|x) \right| \right]\cdot\mathbf{1}_{r(-y'(x),x,\theta)\geq r(y'(x),x,\theta)} \right] \nonumber \\
\hfill &=\mathbb{E}_{x} \left[ \left[\frac{1}{2}+ \left| \frac{1}{2}-p_{\mathrm{CM}}(y'(x)|x) \right| \right]\cdot\mathbf{1}_{r(y'(x),x,\theta)> r(-y'(x),x,\theta)}+ \right. \nonumber \\ \hfill & \left. \left[\frac{1}{2}- \left| \frac{1}{2}-p_{\mathrm{CM}}(y'(x)|x) \right| \right]\cdot(1-\mathbf{1}_{r(y'(x),x,\theta)> r(-y'(x),x,\theta)}) \right] \nonumber \\
\hfill &= \mathbb{E}_{x} \left[ 2\left| \frac{1}{2}-p_{\mathrm{CM}}(y'(x)|x) \right| \cdot\mathbf{1}_{r(y'(x),x,\theta)> r(-y'(x),x,\theta)}\right]+  \nonumber \\ \hfill & \mathbb{E}_x \left[\frac{1}{2}- \left| \frac{1}{2}-p_{\mathrm{CM}}(y'(x)|x) \right| \right]
\end{align}

During normal training, only the sequence of $y'(x)$ is given as a training label for the sample $x$ to the target model. However, a complex model can inform the target model of more information, i.e. $p_{\mathrm{CM}}(y'(x)|x)$ (confidence of the complex model over the training labels). The second term in the right hand side of (\ref{errorbound}) is independent of the choice of $\theta$. This motivates an algorithm to minimize the first term in (\ref{errorbound}). This motivates the following heuristic:

Solve: 
   \begin{align}
   \min \limits_{\theta} \frac{1}{m} \left[ \sum_{i=1}^m \left|\frac{1}{2}- p_{\mathrm{CM}}(y'(x_i)|x_i) \right| r(y'(x_i),x_i,\mathbf{\theta}) \right]
   \end{align}
The above heuristic is motivated by the fact that normal training of a target model (through expected risk minimization) amounts to optimizing $\mathbb{E}_{x}\left[\mathbf{1}_{r(y'(x),x,\theta)> r(-y'(x),x,\theta)} \right]$. 

For the MLE model, replace $r(\cdot)$ by the negative log-likelihood to get an equivalent of the above heuristic.

\section{ResNet Unit Architecture}

The below tables 3 and 4 illustrate the architecture of the complex and target models.

\begin{table}[h!]
 \begin{center}    
    \begin{tabular}{|l|c|} 
    \hline 
      \textbf{Units} & \textbf{Description} \\
      \hline
 Init-conv & $\left[ \begin{array}{c}
 3 \times 3~\mathrm{conv},~16 
\end{array} \right] $\\
\hline
      Resunit:1-0 & $\left[ \begin{array}{c}
 3\times 3~\mathrm{conv},~64 \\
 3\times 3~\mathrm{conv},~64
\end{array}\right]$  \\
 \hline
    (Resunit:1-x)$\times$ 4 & $\left[ \begin{array}{c}
 3\times 3~\mathrm{conv},~64 \\
 3\times 3~\mathrm{conv},~64
\end{array}\right] \times 4$ \\
\hline
      (Resunit:2-0)& $\left[ \begin{array}{c}
 3\times 3~\mathrm{conv},~128 \\
 3\times 3~\mathrm{conv},~128
\end{array}\right] $ \\
 \hline 
   (Resunit:2-x)$\times$ 4& $\left[ \begin{array}{c}
 3\times 3~\mathrm{conv},~128 \\
 3\times 3~\mathrm{conv},~128
\end{array}\right] \times 4 $ \\
   \hline 
       (Resunit:3-0)& $\left[ \begin{array}{c}
 3\times 3~\mathrm{conv},~256 \\
 3\times 3~\mathrm{conv},~256
\end{array}\right] $ \\
\hline
      (Resunit:3-x)$\times$ 4& $\left[ \begin{array}{c}
 3\times 3~\mathrm{conv},~256 \\
 3\times 3~\mathrm{conv},~256
\end{array}\right] \times 4 $ \\
\hline 
    \multicolumn{2}{|c|}{Average Pool}  \\
 \hline 
     \multicolumn{2}{|c|}{Fully Connected - 10 logits} \\
  \hline   
    \end{tabular}
     \end{center}
  \label{tab:cm}
\caption{$18$ unit Complex Model with $15$ ResNet units in CIFAR-10 experiments.}
\end{table}

\begin{table}[h!]
\label{tab:sm}
  \begin{center}    
    \begin{tabular}{|l|c|c|} 
    \hline 
      \textbf{Target Models} & \textbf{Additional Resunits}& \textbf{Rel. Size} \\
      \hline
 TM-1 & None & $\approx$ 1/5\\
\hline
      TM-2 & (Resunit:1-x)$\times 1$ & $\approx$ 1/3\\ & (Resunit:2-x)$\times 1$ & \\
 \hline
    TM-3 & (Resunit:1-x)$\times 2$ &\\
    & (Resunit:2-x)$\times 1$ & $\approx$ 1/2\\
    & (Resunit:3-x)$\times 1$ &\\
\hline
      TM-4 & (Resunit:1-x)$\times 2$ & \\
    & (Resunit:2-x)$\times 2$ & $\approx$ 2/3\\
    & (Resunit:3-x)$\times 2$ &\\
 \hline 
    \end{tabular}
     \caption{Additional Resnet units in the target models apart from the commonly shared ones. The last column shows the approximate size of the target models relative to the complex neural network model in the previous table.}   
  \end{center}
\end{table}

\section{Distillation}
We train the target models using a cross entropy loss with soft targets. Soft targets are obtained from the softmax outputs of the last layer of the complex model (or equivalently the last linear probe) rescaled by temperature $t$ as in distillation of \cite{distill}. By using cross validation, we picked the temperature that performed best on the validation set in terms of validation accuracy for the simple models. We cross-validated over temperatures from the set $\{0.5,3,10.5,20.5,30.5,40.5,50\}$.
See Figures \ref{Fig:distill_val} and \ref{tab:distill} for validation and test accuracies for TM-4 model with distillation at different temperatures.

\begin{figure*}[htbp]
\parbox{0.75\linewidth}{
\centering
\includegraphics[width=9cm]{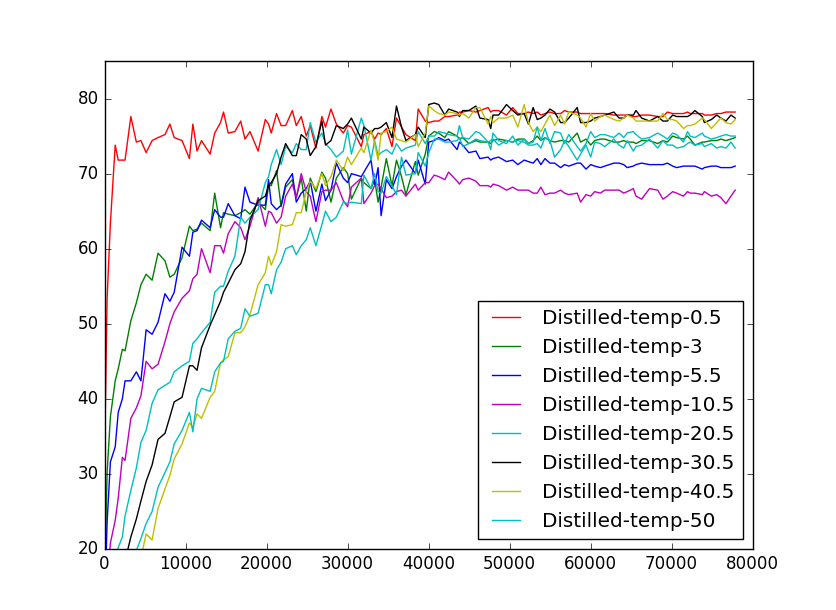}
\caption{Plot of validation set accuracy as a function of training steps for TM-4. The training is done using distillation. Validation accuracies for different temperatures used in distillation are plotted.}
\label{Fig:distill_val}}
\end{figure*}
\begin{figure*}[htbp]
\parbox{.75\linewidth}{
\centering
\begin{tabular}{|c|c|}
\hline
 Distillation Temperatures & Test Accuracy of TM-4 \\
 \hline
 0.5 & 0.7719999990965191 \\
 3.0 & 0.709789470622414 \\
 5.0 & 0.7148421093037254 \\
 10.5 & 0.6798947390757109 \\
 20.5 & 0.7237894786031622 \\
 30.5 & 0.7505263184246264 \\
 40.5 & 0.7513684191201863 \\
 50 & 0.7268421022515548 \\
 \hline
\end{tabular}
\caption{Test Set accuracies of various versions of target model TM-4 trained using distilled final layer confidence scores at various temperatures. The top two are for temperatures $0.5$ and $40.5$. }
\label{tab:distill}
}
\end{figure*}

\bibliographystyle{plain}
\bibliography{FFInterpret} 

\end{document}